\pdfoutput=1
\documentclass{article}

\PassOptionsToPackage{numbers, compress}{natbib}

    \usepackage[preprint]{neurips_2024}

\usepackage[utf8]{inputenc} 
\usepackage[T1]{fontenc}    
\usepackage{hyperref}       
\usepackage{url}            
\usepackage{booktabs}       
\usepackage{amsfonts}       
\usepackage{nicefrac}       
\usepackage{microtype}      
\usepackage{xcolor}         
\usepackage{appendix}
\newcommand{\SM}{\mathbf{SM}}
\newcommand{\DSM}{\mathbf{DSM}}
\newcommand\signfcn{{\mathrm{sign}}}
\newcommand\indicfcn{{\mathbbm{1}}}
\usepackage{amsmath}
\usepackage{amssymb}
\usepackage{mathtools}
\usepackage{wrapfig}
\usepackage{amsthm}
\usepackage{algorithmic}
\usepackage{algorithm}
\usepackage{bbm}
\usepackage{mathrsfs} 
\newcommand{\diam}{\mathscr{D}}
\newtheorem{theorem}{Theorem}[section]

\newtheorem{lemma}[theorem]{Lemma}
\newtheorem{assumption}[theorem]{Assumption}
\newtheorem{definition}[theorem]{Definition}

\title{Analyzing Neural Network-Based Generative Diffusion Models through Convex Optimization}

\author{%
   Fangzhao Zhang\\
  Electrical Engineering\\
  Stanford University\\
  \texttt{zfzhao@stanford.edu} \\
  \And
  Mert Pilanci \\
  Electrical Engineering\\
  Stanford University\\
\texttt{pilanci@stanford.edu} 
}

\begin{document}

\maketitle

\begin{abstract}
  Diffusion models are gaining widespread use in cutting-edge image, video, and audio generation. Score-based diffusion models stand out among these methods, necessitating the estimation of  score function of the input data distribution. In this study, we present a theoretical framework to analyze two-layer neural network-based diffusion models by reframing score matching and denoising score matching as convex optimization. We prove that training shallow neural networks for score prediction can be done by solving a single convex program. Although most analyses of diffusion models operate in the asymptotic setting or rely on approximations, we characterize the exact predicted score function and establish convergence results for neural network-based diffusion models with finite data. Our results provide a precise characterization of what neural network-based diffusion models learn in non-asymptotic settings.
\end{abstract}
\vspace{-1cm}
\section{Introduction}
\vspace{-0.3cm}
\vspace{-0.2cm}
Diffusion models \citep{sohldickstein2015deep} were proposed to tackle the problem of sampling from an unknown distribution and is later shown to be able to generate high quality images  \citep{ho2020denoising}. Song et al.  \cite{song2021scorebased} recognize diffusion model as an example of  score-based models which iteratively exploit Langevin dynamics to produce data from an unknown distribution. This approach only requires the estimation of the score function of the data distribution. Specifically, the simplest form of Langevin Monte Carlo procedure involves first sampling $x^0$ from an initial distribution, then repeating the following steps
\begin{align*}
x^t \leftarrow x^{t-1}+\frac{\epsilon}{2}\nabla_x \log p(x^{t-1}) +\sqrt{\epsilon} z^t,
\end{align*}
where $z^t$ is an independently generated i.i.d. Gaussian noise and $\epsilon$ is a small constant. Here, $\nabla_x \log p(x)$ is known as the score function of the distribution $p(x)$ we desire to sample from. It can be shown that under certain conditions \citep{sinho2023logconcave}, we obtain iterates distributed according to the target distribution $p(x)$ as $\epsilon$ tends to zero and number of iterations tends to infinity. Langevin dynamics sampling procedure suggests that we can attempt to sample from an unknown distribution as long as we can estimate the score function of this distribution at each data point, which is the key observation in current diffusion models designed for generative tasks. In practice, deep neural networks are trained to minimize variants of  score matching objective for fitting the score function. 

Existing literature on the theory of diffusion models typically establish convergence of diffusion process when the learned score function approximates the score of unknown data distribution well, but in reality only empirical approximation is available due to finite training samples and limited neural network (NN) capacity. Current literature falls short in understanding the role of NN approximation error for score-based generative models and it is also difficult to characterize the distribution from which these models sample in practice. However, in \cite{pidstrigach2022scorebased, yi2023generalization, yoon2023diffusion}, the authors show NN-based  score-based generative models given finite training data usually generalize well due to approximation errors introduced by limited model capacity and also optimization errors, recognizing the critical role NN approximation error plays in effectiveness of current large diffusion models. This work contributes to  understanding neural network approximation error in finite  data regime when trained with score matching or denoising score matching objective, which is crucial for understanding neural network-based score-based generative models. Specifically, we answer the following question:
\begin{samepage}
\begin{center}
\textbf{How do NNs approximate the distribution when trained with a (denoising) score matching objective given finite samples and a limited number of neurons?}
\end{center}
\end{samepage}
To summarize, we reframe the (denoising) score matching problem with two-layer neural network as a convex program and characterize the exact form of predicted score function for two-layer neural networks with finite data samples. We establish convergence result for neural network-based Langevin sampling, which serves as a core backbone of nowadays generative models used in application. Our convex program for score matching objective bypasses the Jacobian computation issue for piecewise linear activation function such as ReLU, which stabilizes the training procedure and can have practical benefits. All theoretic findings are corroborated with simulation resutls.
\vspace{-0.2cm}
\section{Background}
\vspace{-0.2cm}
Diffusion model has been shown to be useful in various generative tasks including image generation \citep{ho2020denoising}, audio generation \citep{zhang2023survey}, and text generation \citep{wu2023ardiffusion}. Variants of diffusion models such as denoising diffusion implicit model \citep{song2022denoising} have been designed to speedup sample generation procedure. The key to score-based diffusion model is the estimation of score function at any data point. In practice, a deep neural network model $s_\theta$ is trained to minimize variants of the score matching objective $\mathbb{E}[\|s_\theta(x)-\nabla_x\log p_{\text{data}}(x)\|_2^2]$ and is used for score function estimation. The score matching objective can be shown to be equivalent up to a constant to
\begin{equation}\label{sm_obj_1}
\mathbb{E}_{p_\text{data}(x)}\left[\text{tr}\left(\nabla_x s_\theta(x)\right)+\frac{1}{2}\|s_\theta(x)\|_2^2\right]
\end{equation}
  which is more practical since $\nabla_x\log p_{\text{data}}(x)$ is usually not directly available. To help alleviate the computation overhead in computing trace of Jacobian in (\ref{sm_obj_1}) for deep neural network and high dimensional data, sliced score matching \citep{song2019sliced} which exploits trace estimation method for trace of Jacobian evaluation has been proposed. Another variant used more commonly nowadays in conjunction with annealed Langevin sampling  is denoising score matching \citep{vincent2011connection} which considers sampling from a perturbed distribution and totally circumvents the computation of trace of Jacobian.

  Theory guarantees for diffusion models relate to convergence of log-concave sampling procedure which contains Langevin dynamics. Prior literature establishes convergence of final distribution of log-concave sampling procedure to ground truth data distribution under mild condition with exact score function at each sample point being known \citep{sinho2023logconcave}. Recent work \cite{li2023generalization} characterizes the generalization error for NN-based score models with bounds in number of neurons and obtain vanishing generalization gap when number of neurons tends to infinity, i.e., when the approximation error vanishes. Though neural network approximation error has been recognized to be core to the generalization capability of  diffusion models in deep learning, existing work falls short in characterizing the exact score function learned by neural network with finite samples. Our work focuses on analyzing what neural network-based score model learns in finite regime. Specifically, we show that the score matching objective fitted with two-layer neural network  can be reparametrized as a quadratic convex program and solved directly to global optimality and the predicted score function will be piecewise linear with kinks only at training data points. We also investigate cases where the convex program can be solved analytically and we observe that the predicted score function may not integrate to be concave and thus only convergence to local stationary point is guaranteed.

  Besides theoretic interest mentioned above, our convex programs may have practical benefit since they stabilize the training procedure due to convexity. Moreover, for commonly used activation function such as ReLU, trace of Jacobian involves threshold function which has zero gradient almost everywhere. Therefore, conventional gradient-based optimizers may face difficulties minimizing the training objective. Our convex programs bypass this Jacobian computation issue and thus gain advantage.

  To our best knowledge, this is the first work characterizing the exact score function learned by two-layer neural network with finite data samples and this is also the first convex program derived for score-matching objective. Our work is closely related to prior convex neural network theories \citep{pilanci2020neural,sahiner2021vectoroutput,ergen2023globally} which consider mainly squared loss instead.  A recent work \cite{zeno2023minimumnorm} tackles very similar problems as ours and also studies the NN approximation error in score matching fitting, though there are several key differences which we defer to Appendix \ref{prior_append} for sake of page limitation. In below, we present the convex program derived for score matching objective in Section \ref{sm_section} with exact score characterization and convergence results established. We further delve into the denoising score matching fitting in Section \ref{dsm_section}. We present simulation results with both Langevin Monte Carlo and annealed Langevin sampling with our convex score predictor in Section \ref{simu_sec}. Conclusion and future work is discussed in Section \ref{conclusion}.
  
\textbf{Notation. }We first introduce some notations we will use in later sections. 
 We use $\signfcn(x)$ to denote the sign function taking value $1$ when  $x\in [0,\infty)$ and $-1$ otherwise, and  $\indicfcn$ to denote the $0$-$1$ valued indicator function taking value $1$ when the argument is a true Boolean statement. For any vector $x$, $\signfcn(x)$ and $\indicfcn\{x\geq 0\}$ applies elementwise. We denote the pseudoinverse of matrix $A$ as $A^\dagger$. We denote subgradient of a convex function $f:\mathbb{R}^d\rightarrow\mathbb{R}$ at $x\in\mathbb{R}^d$ as $\partial f(x)\subseteq \mathbb{R}^d$. For any vector $x$, len($x$) denote the dimension of $x.$ Standard asymptotic notation is used, i.e., for any sequence $s_n\in\mathbb{R}$ and any given $c_n\in\mathbb{R}^+$, we use $s_n=\mathcal{O}(c_n)$ and $s_n=\Omega(c_n)$ to represent $s_n<\kappa c_n$ and $s_n>\kappa c_n$ respectively for some $\kappa>0.$ We use $s_n=\Theta(c_n)$ (also $a\asymp b$) when both $s_n=\mathcal{O}(c_n)$   and $s_n=\Omega(c_n)$.
 \vspace{-0.2cm}
 \section{Score Matching}\label{sm_section}
 \vspace{-0.2cm}
In this section, we derive convex program for score matching fitting problem with two-layer neural network and establish convergence results for neural network-based Langevin sampling procedure. We detail the neural network architecture being studied in Section \ref{s3_architecture} and present the corresponding convex program in Section \ref{cvx_theory1}, score prediction characterization and convergence theory is included in Section \ref{convergence}. 
For sake of clarity, we present results for NN without skip connection here in the main content and leave results with more general architecture in Appendix \ref{sec3_more_archi}.
\vspace{-0.2cm}
\subsection{Score Matching Problem and Neural Network Architectures}\label{s3_architecture}
\vspace{-0.2cm}
Let $s_\theta$ denote a neural network parameterized by parameter $\theta$ with output dimension the same as input data dimension, which is required for score matching estimation and is captured by specific U-Net used in nowadays diffusion model implementation. With $n$ data samples, the empirical version of score matching objective (\ref{sm_obj_1}) is 
\begin{equation*}
\SM(s_\theta(x))= \sum_{i=1}^n \mbox{tr}\left(\nabla_{x_i} s_\theta({x_i})\right)+\frac{1}{2}\|s_\theta(x_i)\|_2^2.
\end{equation*}
The final training loss we consider is the above score matching objective together with weight decay term, which writes
\begin{equation}\label{train_obj}
\min_{\theta} ~\SM(s_\theta(x))+\frac{\beta}{2}\|\theta'\|_2^2,
\end{equation}
where $\theta'\subseteq \theta$ denotes the parameters to be regularized. We note that a non-zero weight decay is indeed core for the optimal value to stay finite, see \textcolor{black}{Appendix \ref{unbound_explain} } for explanation, which rationalizes the additional weight decay term involved here. Let $m$ denote number of hidden neurons. Consider   two-layer neural network  of general form  as below
\begin{equation}\label{general_arc}
s_\theta(x)=W^{(2)}\sigma\left(W^{(1)}x+b^{(1)}\right)+Vx+b^{(2)},
\end{equation}
with activation function $\sigma$, parameter $\theta=\{W^{(1)},b^{(1)},W^{(2)}, b^{(2)},V\}$ and $\theta'=\{W^{(1)},W^{(2)}\}$ where $x\in\mathbb{R}^{d}$ is the input data, $W^{(1)}\in\mathbb{R}^{m\times d}$ is the first-layer weight, $b^{(1)}\in\mathbb{R}^{m}$ is the first-layer bias, $W^{(2)}\in\mathbb{R}^{d\times m}$ is the second-layer weight, $b^{(2)}\in\mathbb{R}^{d}$ is the second-layer bias and $V\in\mathbb{R}^{d\times d}$ is the skip connection coefficient. 
\subsection{Convex Programs}\label{cvx_theory1}
\vspace{-0.1cm}
We describe separately the derived convex program for univariate data and multivariate data here since our score prediction characterization and convergence result in below (Section \ref{convergence}) focuses on univariate data and thus presenting the univariate data convex program explicitly helps improve readability.

\textbf{Univariate Data.} Consider training data $x_1,\hdots,x_n\in \mathbb{R}$. Score matching fitting with objective  (\ref{train_obj}) is equivalent to solving a quadratic convex program in the sense that both problems have same optimal value and an optimal NN parameter set which achieves minimal loss can be derived from the solution to the corresponding convex program. We detail this finding in the following theorem,
\begin{theorem}\label{thm3}
When $\sigma$ is ReLU or absolute value activation and $V=0$, denote the optimal score matching objective value (\ref{train_obj}) with $s_\theta$ specified in (\ref{general_arc})  as $p^*,$ when $m\geq \text{len}(y)$ and $\beta\geq  1\footnote{Note when $\beta<1$, the optimal value to problem  (\ref{train_obj}) may be unbounded, see \textcolor{black}{Appendix \ref{unbound_explain} } for explanation.},$  
\begin{equation}\label{relu_noskip_formula}
\begin{aligned}
p^*=& \min_{y}\quad \frac{1}{2}\|Ay\|_2^2+b^Ty+\beta\|y\|_1\,,
\end{aligned}
\end{equation}
where the entries of $A$ are determined by the pairwise distances between data points, and the entries of $b$ correspond to the derivative of $\sigma$ evaluated at entries of $A$ (see Appendix \ref{thm31_proof} for the formulas).
\end{theorem}
\vspace{-0.2cm}
\begin{proof}
See Appendix \ref{thm31_proof}.
\end{proof}
\vspace{-0.2cm}
More precisely, when $\sigma$ is absolute value function, quadratic term coefficient $A\in\mathbb{R}^{n\times 2n}$  is formed by first taking pairwise distance between data points as $[A_1]_{ij}=|x_i-x_j|,$ then we normalize $A_1$ column-wise with mean reduced to form $\bar A_1=\left(I-\frac{1}{n}11^T\right)A_1$, the desired $A$ is simply a concatenation of two copies of $\bar A_1$ as $A=[ \bar A_1,\bar A_1].$ Linear term coefficient  $b\in\mathbb{R}^{2n}$ is column sum of two $n\times n$ matrices $[C_1]_{ij}=\signfcn(x_i-x_j)$ and $[C_2]_{ij}=\signfcn(-x_i+x_j)$ and formally writes $b=[1^TC_1,-1^TC_2]^T$. Once an optimal solution $y^\star$ to the convex program (\ref{relu_noskip_formula}) has been derived, we can reconstruct an optimal NN parameter set $\{{W^{(1)}}^\star,{W^{(2)}}^\star,{b^{(1)}}^\star,{b^{(2)}}^\star\}$ that achieves minimal training loss simply from data points $\{x_1,\cdots,x_n\}$ and $y^\star.$ See Appendix \ref{thm31_proof} for the reconstruction procedure. Given all this, with $y^\star$ known, for any test data $\hat x$, the predicted score is  given by
\[\hat y(\hat x)=\sum_{i=1}^{n} (|y^\star_i|+|y^\star_{i+n}|)|\hat x-x_i|+b_0^\star,\]
where $b_0^\star=-\frac{1}{n}1^T([A_1,A_1]y^\star).$ Remarkbly,  the optimal score is a  piecewise linear function with breakpoints only at a subset of data points. When training data points are highly separated, the optimal score approximately corresponds to the score function of a mixture of Gaussians with centroids at $\{\hat x : \hat y(\hat x)=0\}$. The breakpoints delineate the ranges of each Gaussian component. 



\textbf{Multivariate Data.} To state the convex program for multivariate data, we first introduce the concept of arrangement matrices.  When $d$ is arbitrary, for data matrix $X\in\mathbb{R}^{n\times d}$  and any arbitrary vector $u\in\mathbb{R}^d$, We consider the set of diagonal matrices 
 \[\mathcal{D}:=\{\text{diag}(\indicfcn\{Xu\geq 0\})\},\]
 which takes value $1$ or $0$ along the diagonal that indicates the set of possible arrangement activation patterns for the ReLU activation. Indeed, we can enumerate the set of sign patterns as $\mathcal{D}=\{D_{i}\}_{i=1}^{P}$ where $P$ is bounded by
 \[P\leq 2r\left(\frac{e(n-1)}{r}\right)^r\]
for $r=\text{rank}(X)$ \citep{pilanci2020neural,stanley2004introduction}. Since the proof of Theorem \ref{thm3} is closely tied to reconstruction of optimal neurons and does not trivially extend to multivariate data, we instead build on \cite{mishkin2022fast} and employ an alternative duality-free proof to derive our conclusion for multivariate data. The result holds for zero  $b^{(1)},b^{(2)},V$ and $\beta$, i.e., when there is no bias term, skip connection, and weight decay added. We present here result for ReLU $\sigma$. See Appendix \ref{thm32_proof} for conclusion for the case when $\sigma$ is absolute value  activation. 
\begin{theorem} \label{sm_nd_thm}
When $\sigma$ is ReLU, $b^{(1)},b^{(2)},V$ and 
 $\beta$ all zero,  denote the optimal score matching objective value (\ref{train_obj}) with $s_\theta$ specified in (\ref{general_arc}) as $p^*$, when $m\geq 2Pd$, under \textcolor{black}{ Assumption \ref{multi_assum}},
\begin{equation}\label{high_dim_cvx0}
p^*=\min_{W_i} \frac{1}{2} \left\|\sum_{i=1}^P D_iX W_i\right\|_F^2+\sum_{i=1}^P \text{tr}(D_i)\text{tr}(W_i).
\end{equation}
\end{theorem}
\vspace{-0.2cm}
\begin{proof}
See \textcolor{black}{Appendix \ref{thm32_proof}}.
\end{proof}
\vspace{-0.2cm}
  Prior work \citep{JMLR:v6:hyvarinen05a, lin2016estimation} observes that with linear activation, the optimal weight matrix of score fitting reduces to  empirical precision matrix which models the correlation between  data points and the authors exploit this fact in graphical model construction. Here we show that the optimal $W_i$'s solved for (\ref{high_dim_cvx0}) correspond to piecewise empirical covariance estimator and therefore the non-linear two-layer NN is a more expressive model compared to prior linear models. To see this, we first write $\tilde V=[\mbox{tr}(D_1)I,\mbox{tr}(D_2)I,\ldots,\mbox{tr}(D_P)I], W=[W_1,\ldots,W_P]^T,\tilde X=[D_1X,\ldots,D_PX]$, then the convex program (\ref{high_dim_cvx0}) can be rewritten as
\begin{equation}\label{highdim-cvx3}
\min_W \frac{1}{2}\big\|\tilde X W\big\|_F^2+ \langle \tilde V, W \rangle.
\end{equation}
When the optimal value is finite, e.g., $\tilde V \in \mathrm{range}(\tilde X^T \tilde X)$, an optimal solution to (\ref{highdim-cvx3}) is given by
\begin{equation*}
\begin{aligned}
W^\star&=-(\tilde X^T\tilde X)^{\dagger}\tilde V\\
&=-\begin{bmatrix}\sum_{k\in S_{11}}X_kX_k^T & \sum_{k\in S_{12}} X_kX_k^T &\cdots\\\sum_{k\in S_{21}}X_kX_k^T&\sum_{k\in S_{22}}X_kX_k^T&\cdots \\ &\cdots& \end{bmatrix}^{\dagger}\begin{bmatrix}\mbox{tr}(D_1)I\\\mbox{tr}(D_2)I\\\vdots\\\mbox{tr}(D_P)I\end{bmatrix},
\end{aligned}
\end{equation*}
where $S_{ij}=\{k: X_k^Tu_i\geq 0, X_k^Tu_j\geq 0\}$ and $u_i$ is the generator of $D_i=\text{diag}(\indicfcn\{Xu_i\geq 0\})$. The above expression for $W^\star$ can be seen as a (negative) piecewise empirical covariance estimator which partitions the space with hyperplane arrangements. When $P=1$ and $D_1=I$ , $|W^\star|=(\tilde X^T\tilde X)^\dagger$ reduces to the empirical precision matrix corresponding to linear activation model.
\vspace{-0.2cm}
\subsection{Score Prediction and Convergence Result}\label{convergence}
\vspace{-0.2cm}
In this section, we delve into the convex program (\ref{relu_noskip_formula}) and show that with distinct data points and large weight decay, (\ref{relu_noskip_formula}) can be solved analytically and the integration of predicted score function is always concave for ReLU activation, which aligns with theoretic assumptions for Langevin sampling procedures. We then establish convergence result for Langevin dynamics in this regime.  Though the same observation does not persist for absolute value activation, where the predicted score function may integrate to be non-concave and thus only convergence to stationary points holds. All notations in this section follow Section \ref{cvx_theory1}.

\textbf{Score Prediction. }Consider the case when $\sigma$ is ReLU and $V=0$, let $\mu=\sum_{i=1}^n x_i/n$ denote the sample mean and  $v=\sum_{i=1}^n (x_i-\mu)^2/n$ denote the sample variance. We know $A\in\mathbb{R}^{n\times 4n}$ and $y\in\mathbb{R}^{4n}$ following Appendix \ref{thm31_proof}. When $\beta> \|b\|_\infty,y^\star=0$ is optimal and the neural network will always predict zero score no matter what input it takes. When $\beta_1< \beta\leq \|b\|_\infty$ for some threshold $\beta_1$\footnote{See \textcolor{black}{Appendix \ref{relu_noskip_score}} for value of $\beta_1$.},   $y^\star$ is all zero except for the first and the  $3n$-th entry, which have value $(\beta-n)/2nv+t$ and $(n-\beta)/2nv+t$ for some $|t|\leq \frac{n-\beta}{2nv}$ respectively \footnote{See \textcolor{black}{Appendix \ref{relu_noskip_score}} for proof.}. Therefore, for any input data point $\hat x,$ the predicted score $\hat y$ is
\begin{equation}\label{pred_score}
\begin{cases}
\hat y=\frac{\beta-n}{nv}(\hat x-\mu),  ~~\qquad \qquad\qquad \qquad\quad x_1\leq \hat x  \leq x_n\\
\hat y=-(\frac{n-\beta}{2nv}+t)\hat x+(\frac{\beta-n}{2nv}+t)x_1+\frac{n-\beta}{nv}\mu,  ~\quad \hat x<x_1\\
\hat y=(\frac{\beta-n}{2nv}+t)\hat x-(\frac{n-\beta}{2nv}+t)x_n+\frac{n-\beta}{nv}\mu, \qquad \hat x>x_n
\end{cases}
\end{equation}
\begin{figure*}[h]
 \centering
\includegraphics[width=0.9\linewidth]{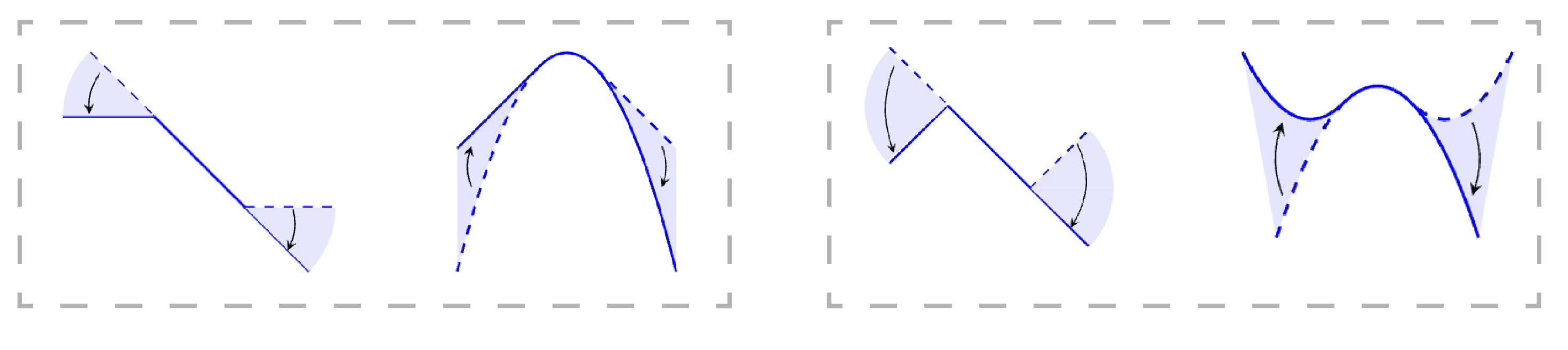}
\caption{Predicted score function and its integration for univariate data with two-layer neural network with ReLU activation (left) and absolut value activation (right).  The left subplot shows all optimal score predictions  by convex score predictor  for univariate input data of arbitrary distribution for certain weight decay range and the right subplot shows its integration. See Section \ref{convergence}  for details.}\label{relu_abs}
\vspace{-0.2cm}
\end{figure*}
for some $|t|\leq \frac{n-\beta}{2nv}.$ Left plot in Figure \ref{relu_abs} provides a visualization of (\ref{pred_score}) and its integration. Note within sampled data range, the predicted score function  aligns with score function of Gaussian distribution parameterized by sample mean $\mu$ and sample variance  $v$; outside data range, the predicted score function is a linear interpolation. The integration of score function is always concave in this case, and therefore Langevin dynamics sampling with predicted score function has well-established convergence guarantees \citep{Durmus2016SamplingFA, dalalyan2016theoretical, durmus2016nonasymptotic}.

Contrarily, when $\sigma$ is absolute value activation and still $V=0$, when $\beta$ goes below $\|b\|_\infty$ and stays above some threshold $\beta_2$ (see Appendix \ref{abs_noskip_score} for details), for any test data $\hat x$, the corresponding predicted score $\hat y$ is given by
\begin{equation*}
\begin{cases}
\hat y=\frac{\beta-n}{nv}(\hat x-\mu),\quad & x_1\leq \hat x\leq x_n,\\
\hat y=-2t\hat x+(\frac{\beta-n}{nv}+2t)x_1+\frac{n-\beta}{nv}\mu, \quad & \hat x<x_1,\\
\hat y=2t\hat x-(\frac{n-\beta}{nv}+2t)x_n+\frac{n-\beta}{nv}\mu,\quad & \hat x>x_n.
\end{cases}
\end{equation*}
Right plot in Figure \ref{relu_abs} depicts the score prediction and its integration. Within the sampled data range,  the score prediction corresponds to score of Gaussian distribution parameterized by sample mean and sample variance which is the same as the score predicted by ReLU neural network. The score prediction outside sampled data range is still a linear interpolation but with a different slope from what is  predicted by the ReLU neural network. This underscores the distinction between absolute value activation and ReLU activation.   The corresponding probability density when $\sigma$ being absolute value activation is log-concave only when $t=0$. Notably, the solution with $t=0$ corresponds to the unique minimum norm solution of the convex program (\ref{relu_noskip_formula}), highlighting its significance. Here the score prediction no longer corresponds to log-concave distribution except for the min-norm case and classic convergence theory has only theoretic assurance for converging to stationary points \cite{sinho2023logconcave}.

When skip connection is added, i.e., $V\neq 0,$ the optimal score prediction corresponding to $\beta>\|b\|_\infty$ is no longer zero and the corresponding optimal neural network parameter set is given by $\{W^{(1)}=0,b^{(1)}=0,W^{(2)}=0,b^{(2)}=\mu/v,V=-1/v\}$. For any test data $\hat x$, the predicted score is given by
\[\hat y=V\hat x +b^{(2)}=-\frac{1}{v}(\hat x-\mu),\]
which aligns with the score function of Gaussian distribution with mean being sample mean and variance being sample variance.Therefore, adding skip connection would change the zero score prediction to a linear function parameterized by sample mean and variance in the large weight decay regime.  See Appendix \ref{relu_skip_score} and \ref{abs_skip_score} for details.
\begin{center}
\begin{minipage}{0.46\textwidth}
\begin{algorithm}[H]
    \caption{Score Matching \label{alg_1}}
\begin{algorithmic}
   \STATE {\textbf{Input:}} training data $x_1,\hdots,x_n\in\mathbb{R}^d$
   \STATE minimize
   \[ \sum_{i=1}^n \frac{1}{2}s_\theta^2(x_i)+ \nabla_\theta s_\theta(x_i)+\frac{\beta}{2}\|\theta'\|_2^2\]
\end{algorithmic} 
\end{algorithm}
\end{minipage}
\hfill
\begin{minipage}{0.5\textwidth}
\begin{algorithm}[H]
      \caption{Langevin Monte Carlo \label{alg_2}}
   \label{alg:newton}
\begin{algorithmic}
\STATE {$~$}
   \STATE {\bfseries Initialize:} 
   $x^0\sim \mu_0(x)$
   \FOR{$t=1,2,...,T$}
   \STATE $z^t\sim \mathcal{N}(0,1)$
   \STATE $x^t\leftarrow x^{t-1}+\eta s_\theta(x^{t-1})+\sqrt{2\eta}z^t$
   \ENDFOR
\end{algorithmic}
\end{algorithm}
\end{minipage}

\end{center}
\textbf{Convergence Result.} Here we state our convergence result for Langevin sampling with NN-based score predictor.  Strong convergence guarantees for  Langevin Monte Carlo method are often contingent upon the log-concavity of the target distribution. Consider two-layer ReLU network without skip connection and consider training points $x_1,\ldots,x_n\in\mathbb{R}$. We have derived that the NN-predicted score function for any input distribution is always concave given $\beta_1<\beta\leq \|b\|_\infty$, thus we can exploit existing convergence results for log-concave sampling to derive the convergence of  Langevin dynamics with a neural network-based score function, which we state formally as the below theorem.
\begin{theorem}\label{conv_thm2}
When $s_\theta$ used in Algoritm \ref{alg_2} is of two-layer ReLU (without skip connection) trained to optimal with Algorithm \ref{alg_1} and $\textcolor{black}{\beta_1}<\beta< n,$ let $\pi$ denote the target distribution (defined below). In Algorithm \ref{alg_2}, for any $\epsilon\in [0,1]$, if we take step size $\eta\asymp \frac{\epsilon^2 nv}{n-\beta},$ then for  $\overline\mu=T^{-1}\sum_{t=1}^T x^t$, it holds that $\sqrt{\text{KL}(\overline \mu\|\pi)}\leq\epsilon$ after
\[O\left(\frac{(n-\beta) W_2^2(\mu_0,\pi)}{nv\epsilon^4}\right)\quad \text{iterations},\]
where $W_2$ denotes 2-Wasserstein distance and $\pi$ satisfies   
\[
\pi\propto
\begin{cases}
\exp(\frac{\beta-n}{2nv}x^2-\frac{\mu(\beta-n)}{nv}x), \qquad \qquad x_1\leq x\leq x_n,\\
\exp((\frac{\beta-n}{4nv}-\frac{t}{2})x^2+(\frac{\beta-n}{2nv}+t)x_1x\\\qquad+(\frac{\mu(n-\beta)}{nv})x+(\frac{n-\beta}{4nv}-\frac{t}{2})x_1^2), \qquad x<x_1,\\
\exp((\frac{\beta-n}{4nv}+\frac{t}{2})x^2-(\frac{n-\beta}{2nv}+t)x_nx\\\qquad +(\frac{\mu(n-\beta)}{nv})x+(\frac{t}{2}+\frac{n-\beta}{4nv})x_n^2),\qquad  x> x_n,
\end{cases}
\]
for some $|t|\leq \frac{n-\beta}{2nv}$.
\end{theorem}
\begin{proof}
See \textcolor{black}{Appendix \ref{conv_proof}}.
\end{proof}
To the best of our knowledge, prior to our study, there has been no characterization of the sample distribution generated by Algorithm \ref{alg_2} when the score model is trained using Algorithm \ref{alg_1}.
\vspace{-0.2cm}
\section{Denoising Score Matching}\label{dsm_section}
\vspace{-0.2cm}
To tackle the difficulty in computation of trace of Jacobian required in score matching objective (\ref{sm_obj_1}),  denoising score matching has been proposed in \citep{vincent2011connection}. It then becomes widely used in practical generative models, especially for its natural conjunction with annealed Langevin sampling procedure, which forms the current mainstream noising/denoising paradigm of large-scale diffusion models being used in popular applications. In this section, we reframe denoising score matching fitting problem with two-layer neural network as a convex program which can be solved to global optimality stably. We empirically verify the validity of our theoretic findings in Section \ref{simu_sec}.

To briefly review, denoising score matching  first perturbs data points  with a predefined noise distribution and then estimates the score of the perturbed data distribution. When the noise distribution is chosen to be standard Gaussian, for some noise level $\epsilon>0$, the objective is equivalent to 
\begin{equation*}
\min_\theta~ \mathbb{E}_{x\sim p_\text{data}(x)}\mathbb{E}_{\delta\sim \mathcal{N}(0,I)} \left\|s_\theta(x+\epsilon \delta)-\frac{\delta}{\epsilon}\right\|_2^2,
\end{equation*}
with the empirical version given by
\begin{equation}\label{dddsm}
\DSM(s_\theta)=\sum_{i=1}^n \frac{1}{2}\left\|s_\theta(x_i+\epsilon \delta_i)-\frac{\delta_i}{\epsilon}\right\|_2^2,
\end{equation}
where $\{x_i\}_{i=1}^n$ are samples from $p_\text{data}(x)$ and $\{\delta_i\}_{i=1}^n$ are samples from standard Gaussian. The final training loss we consider is the above score matching objective together with weight decay term, which writes
\begin{equation}\label{dsm_obj}
\min_{\theta} ~\DSM(s_\theta(x))+\frac{\beta}{2}\|\theta'\|_2^2,
\end{equation}
where $\theta'\subseteq \theta$ denotes the parameters to be regularized. Unlike for score matching objective where weight decay is important for optimal objective value to stay finite, here for denoising objective, weight decay is unnecessary and can be removed. In our derived convex program, we allow $\beta$ to be arbitrarily close to zero so the result is general. Note (\ref{dsm_obj}) circumvents the computation of trace of Jacobian  and is thus more applicable for training tasks in large data regime. We consider the same neural network architecture  described in Section \ref{s3_architecture} except that here we only consider case for $V=0$. Like in Section \ref{cvx_theory1}, we still present our conclusion for univariate data and multivariate data separately so that we can easily demonstrate deeper investigation on univariate data findings.

\textbf{Univariate Data.} Consider training data $x_1,\hdots,x_n\in \mathbb{R}$. Denoising score matching fitting with objective  (\ref{dsm_obj}) is equivalent to solving a lasso problem in the sense that both problems have same optimal value and an optimal NN parameter set which achieves minimal loss can be derived from the solution to the corresponding lasso program. The difference between convex program of denoising score matching fitting and that of score matching fitting is that no linear term is included in this scenario. We detail this finding in the following theorem,
\begin{theorem}\label{thm4}
When $\sigma$ is ReLU or absolute value activation and $V=0$,   denote the optimal denoising score matching objective value (\ref{dsm_obj}) with $s_\theta$ specified in (\ref{general_arc}) as $p^*,$ when $m\geq \text{len}(y)$ and $\beta>  0$,  
\begin{equation}\label{thm4formula}
\begin{aligned}
p^*=& \min_{y}\quad \frac{1}{2}\|Ay+b\|_2^2+\beta\|y\|_1\,,
\end{aligned}
\end{equation}
where the entries of $A$ are determined by the pairwise distances between data points.
\end{theorem}
\vspace{-0.2cm}
\begin{proof}
See \textcolor{black}{Appendix \ref{thm41_proof}}.
\end{proof}
\vspace{-0.2cm}
For demonstration, consider $\sigma$ being ReLU, then coefficient matrix $A\in\mathbb{R}^{n\times 2n}$ is concatenation of two $n\times n$ matrices, i.e., $A=[\bar A_1, \bar A_2]$ where $\bar A_1=(I-11^T/n)A_1$ is the column-mean-subtracted version of $A_1$ and $[A_1]_{ij}=(x_i-x_j)_+$ measures pairwise distance between data points. Similarly, we have $\bar A_2=(I-11^T/n)A_2$ with $[A_2]_{ij}=(-x_i+x_j)_+.$ Label vector $b$ is the mean-subtracted version of original training label $l=[\sigma_1/\epsilon,\sigma_2/\epsilon,\cdots,\sigma_n/\epsilon]^T.$ Once an optimal $y^\star$ to (\ref{thm4formula}) has been derived, we can construct an optimal NN parameter set that achieves minimal training loss out of $y^\star$ and data points. See Appendix \ref{thm41_proof} for details. Under this construction, with value of $y^\star$ known, given any test data $\hat x$, NN-predicted score is
\[\hat y(\hat x)=-\sum_{i=1}^n y_i^\star(x-x_i)_+-\sum_{i=1}^n y_{n+i}^\star(-x+x_i)_++b_0^\star,\]
with $b_0^\star=\frac{1}{n}1^T([A_1,A_2]y^\star+l)$. We then proceed to present our multivariate data result, which holds for $b^{(1)},b^{(2)}$ and $\beta$ being all zero due to a change of our proof paradigm.

\textbf{Multivariate Data.} Let $L\in\mathbb{R}^{n\times d}$ denote the label matrix, i.e., $L_i=\delta_i/\epsilon$, and $\mathcal{D}=\{D_{i}\}_{i=1}^{P}$ be the arrangement activation patterns for ReLU  activation  as defined in Section \ref{cvx_theory1}, we have the following result for ReLU $\sigma$. See Appendix \ref{thm42_proof} for also convex program for absolute value activation.
\begin{theorem} \label{thm4_2}
When $\sigma$ is ReLU, $b^{(1)},b^{(2)},V$ and 
 $\beta$ all zero,  denote the optimal denoising score matching objective value (\ref{dsm_obj}) with $s_\theta$ specified in (\ref{general_arc}) as $p^*$, when $m\geq 2Pd$, under \textcolor{black}{ Assumption \ref{multi_assum}},
\begin{equation}\label{high_dim_cvx_2}
p^*=\min_{W_i} \frac{1}{2} \left\|\sum_{i=1}^P D_iX W_i-L\right\|_F^2.
\end{equation}
\end{theorem}
\vspace{-0.2cm}
\begin{proof}
See Appendix \ref{thm42_proof}.
\end{proof}
\vspace{-0.2cm}
The derived convex program (\ref{high_dim_cvx_2}) is a simple least square fitting and any convex program solver can be used to solve it efficiently.

\vspace{-0.2cm}
\section{Numerical Results}\label{simu_sec}
\vspace{-0.2cm}
Finally, we corroborate our previous findings with synthetic data simulation. For sake of page limitation, we present here in main text some of our simulation results for score matching fitting (Section \ref{cvx_theory1})  with Langevin sampling and denoising score matching fitting (Section \ref{dsm_section})  with annealed Langevin sampling. We defer more empirical observations to Appendix \ref{simu_supp}.
\vspace{-0.2cm}
\subsection{Score Matching Simulations}\label{score_simu_sec}
\vspace{-0.2cm}
\begin{figure*}[ht!]
  \makebox[\textwidth][c]{\includegraphics[width=1.0\linewidth]{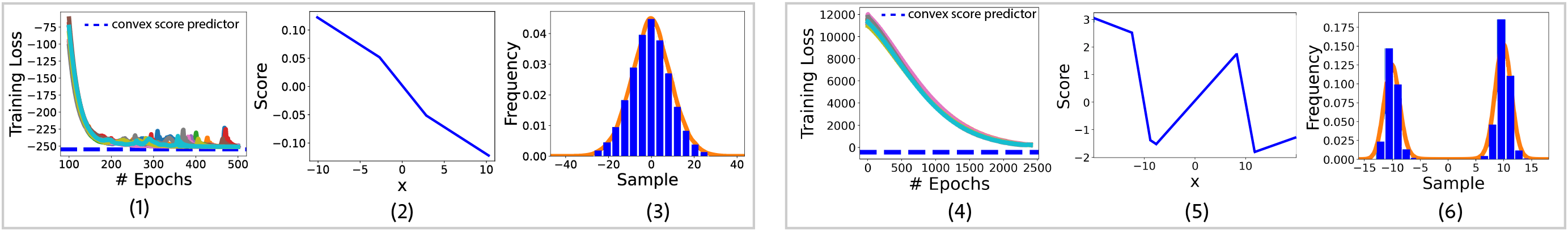}}
\vspace{-0.2cm}
\caption{Simulation results for 
 score matching tasks  with two-layer ReLU neural network. Left figure is for Gaussian data, right figure is for two-component Gaussian mixture. Sampling histogram is with Langevin dynamics. See Section \ref{score_simu_sec} for details.}\label{score_simu_}
 \vspace{-0.2cm}
\end{figure*}
For score matching fitting problems, we verify both the validity of our convex program (Equation \ref{relu_noskip_formula}) and our score prediction characterization (Figure \ref{relu_abs}) with univariate Gaussian data and we show that the derived convex program (Equation \ref{relu_noskip_formula}) is also able to capture two-component Gaussian mixture distribution. We also present sampling histograms with Langevin dynamics (non-annealed) aided by our convex score predictor. For univariate Gaussian data simulation, we set $\beta=\|b\|_\infty-1$. Plot (1) in Figure \ref{score_simu_} compares objective value of non-convex training with Adam optimizer and our convex program loss solved via CVXPY \cite{diamond2016cvxpy}. The dashed blue line denotes our convex program objective value which solves the training problem globally and stably. Plot (2) is for score prediction, which  verifies our analytical characterization in Section \ref{convergence} and aligns with Figure \ref{relu_abs}. Plot (3) shows sampling histogram via Langevin dynamics which recognizes the underline Gaussian as desired. The right figure in Figure \ref{score_simu_} repeats the same experiments for two-component Gaussian mixture distribution with a slightly small $\beta$ value since we known from Section \ref{convergence} that $\beta=\|b\|_\infty-1$ cannot capture Gaussian mixture distribution. Our convex program identifies the underline distribution accurately. See Appendix \ref{sm_simu_detail} for more experimental details.
\vspace{-0.2cm}
\subsection{Denoising Score Matching  Simulations}\label{dsm_simu_sec}
\vspace{-0.2cm}
\begin{figure*}[ht!]
 \makebox[\textwidth][c]{\includegraphics[width=1.0\linewidth]{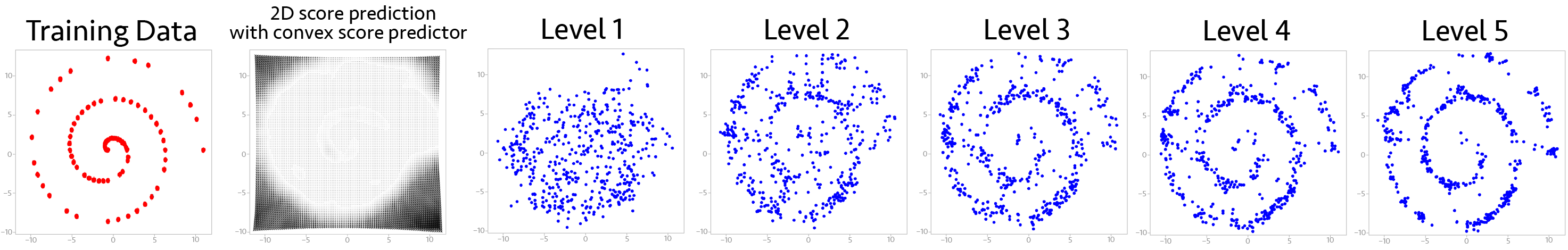}}
\vspace{-0.2cm}
\caption{2D simulation results for 
 denoising score matching tasks with our convex score predictor. The second figure shows vector field plot for score predicted by our convex score predictor. The right plots show denoising procedure with different noise levels in annealed Langevin sampling. See Section \ref{dsm_simu_sec} for details.}\label{dscore_simu_}
 \vspace{-0.1cm}
\end{figure*}
For denoising score matching fitting problems, we verify our derived program (\ref{high_dim_cvx_2}) for $2d$ spiral data  and present sampling results with annealed Langevin process integrated with our convex score predictor. For easier computation, we switch to a variant of program (\ref{high_dim_cvx_2}) in our implementation, see Appendix \ref{2d_exact} for details. The left plot in Figure \ref{dscore_simu_} shows the spiral training data  and the second left plot depicts the score predicted by our $2d$ convex score predictor solved with CVXPY \cite{diamond2016cvxpy}. It can be clearly observed that the score prediction already aligns with the training data shape.The five right plots in Figure \ref{dscore_simu_} are sampling results with annealed Langevin process aided with our $2d$ convex score predictor after different levels of denoising, see Appendix \ref{dsm_simu_detail} for experimental details. Our convex program for denoising score matching works well in capturing training data distribution.
\vspace{-0.2cm}
\section{Conclusion and Future Work}\label{conclusion}
\vspace{-0.2cm}
In this work, we analyze neural network-based diffusion models via  lens of convex optimization. We derive equivalent convex programs for (denoising) score matching fitted with two-layer neural network. For Langevin dynamics with NN-based score predictor, we first characterize the score prediction by solving the derived convex program analytically and we then establish convergence result based on existing convergence theory for log-concave sampling procedure. Notably for univariate data, for certain weight decay range, the predicted score would capture Gaussian distribution characterized by sample mean and sample variance no matter what input distribution is and for general weight decay, score prediction aligns with Gaussian mixture distribution when training data is highly separated. Besides theoretic interest, the derived convex program has potential empirical benefits since it bypasses the difficulty of using gradient-based optimizers due to the Jacobian terms. Our theoretic findings are corroborated with simulation results. For future work,  our proof technique can be extend easily to networks of arbitrary depth by considering convex reparameterizations (see e.g., \cite{ergen2023path, wang2023parallel}).

\newpage
\section{Acknowledgement}
This work was supported in part by the National Science Foundation (NSF) under Grant DMS-2134248; in part by the NSF CAREER Award under Grant CCF-2236829; in part by the U.S. Army Research Office Early Career Award under Grant W911NF-21-1-0242; and in part by the Office of Naval Research under Grant N00014-24-1-2164.
\bibliography{reference}
\bibliographystyle{abbrv}
\newpage
\section*{Appendix}
\begin{appendix}
\section{More on Prior Work}\label{prior_append}
Here we make a note on difference between our work and  work \cite{zeno2023minimumnorm}, which  tackles very similar problems as ours though there are several key differences. In \citep{zeno2023minimumnorm}, the authors study shallow neural network trained for score denoisers and characterize the exact neural network output. The authors show contractive property for NN-based denoiser and prove NN-based denoiser is advantageous against eMMSE denoiser. In our work, we study the exact score matching objective (\ref{sm_obj_1}) which has not been considered in the other  work and establish convergence result for NN-based score predictor which will be much harder to prove for NN-based denoiser due to involvement of noise. Moreover, for denoising score matching objective, we derive convex programs for arbitrary weight decay for multivariate data  while the characterization in \cite{zeno2023minimumnorm} is for vanishing weight decay. For multivariate data, the authors of \cite{zeno2023minimumnorm} only consider  modified objective with data belongs to special subspaces while our convex program holds in general. Finally, our analysis is based on convex optimization theory and no convexification is considered in \cite{zeno2023minimumnorm}.   Our work can be viewed as complementary to \cite{zeno2023minimumnorm} in the sense that we study similar problems with   different objectives and constraints from different angles. 

Another work \cite{han2024neural} establishes approximation error bound between true score function $\nabla \log p(X)$ and the GD minimizer to score-matching objective, which also serve as an approximation error bound between our convex program and true score in some cases while our method bypasses the potential local optimum problem caused by GD and derives the exact score function being predicted. The final result in  \cite{han2024neural} is presented as error bound on predicted score and true score asymptotically in number of hidden neurons $m$ and number of data samples $N$ with expectation over noises added to data and random initialization. Our work doesn't have such a bound while we can solve the training problem globally and thus escape local minimum with a convex program (which only holds with Assumption 3.8 in the other work) and derive the score function for finite data and neurons, i.e., we know exactly what's the predicted score function analytically (see Section \ref{convergence} in our work) in certain regime while the other work only has an error bound on this. 
\section{Explanation for Unbounded Objective Value}\label{unbound_explain}
Here we illustrate via a simple example that weight decay is necessary for the optimal objective value to stay finite. Follow notation in Section \ref{sm_section}, 
 consider for example only one data point and one hidden neuron, then objective function for the neural network with ReLU activation and no skip connection would be 
\[\frac{1}{2}((xw+b)_+\alpha+b_0)^2+w\alpha \indicfcn(xw+b\geq 0)+\frac{\beta}{2}(w^2+\alpha^2). \]
WLOG consider $x=1,$ then when weight decay parameter $0\leq \beta <1,$ set $b=-w+\sqrt{1-\beta}$ and $b_0=0$ above, we get 
\[\frac{1-\beta}{2}\alpha^2+w\alpha+\frac{\beta}{2}(w^2+\alpha^2).\]
Then we can set $\alpha=-w$ and the above expression becomes $(\beta-1)w^2/2$. Thus the objective goes to minus infinity when $w$ goes to infinity.
\section{Proof in Section \ref{cvx_theory1}}
\subsection{Technical Lemmas}
\begin{lemma}\label{lemma_2}
The below constraint set is strictly feasible only when $\beta>1 .$
\[\begin{cases}
|z^T(x-1x_i)_+-1^T\indicfcn\{x-1 x_i\geq 0\}|\leq \beta \quad \\
|z^T(x-1 x_i)_+-1^T\indicfcn\{x-1x_i>0\}|\leq \beta \quad \\
|z^T(-x+1x_i)_++1^T\indicfcn\{-x+1x_i\geq 0\}|\leq \beta \quad \\
|z^T(-x+1x_i)_++1^T\indicfcn\{-x+1 x_i>0\}|\leq \beta \quad \\
z^T1 = 0
\end{cases} \forall i=1,\cdots,n\] 
\begin{proof}
Consider without loss of generality that $x_1<x_2<\cdots<x_n.$  Let $k=-\sum_{j=1}^m z_j(x_j-x_i)+m$ for some $1\leq m\leq n$, the first four constraints with $i=m$ are then $|z^Tx-(n+1)+k|, |z^Tx-(n+1)+k+1|,|k|,|k-1|$. When $i=n,$ the first constraint is $\beta\geq 1$. Thus $\beta>1$ is necessary for the constraint set to be strictly feasible. Since we can always find $z^\star$ satisfying
\[\begin{cases}
x^Tz^\star=n\\
1^Tz^\star=0\\
(x-1x_i)_+^Tz^\star-1^T\indicfcn\{x-1x_i\geq 0\}=0\quad \forall i=2,\cdots,n-1
\end{cases}\]
Note such $z^\star$ satisfies all constraints in the original constraint set when $\beta>1$. Therefore when $\beta>1,$ the original constraint is strictly feasible. 
\end{proof}
\end{lemma}

\begin{lemma}
\label{lemma_1}
The below constraint set is strictly feasible only when $\beta>1.$
\[\begin{cases}
|z^T|x-1 x_i|-1^T\signfcn(x-1x_i)|\leq\beta\\
|z^T|-x+1x_i|+1^T\signfcn(-x+1x_i)|\leq\beta\\
z^T1=0
\end{cases} \forall i=1,\cdots,n\] 
\end{lemma}
\begin{proof}
Consider without loss of generality that $x_1<x_2<\cdots<x_n.$ Then taking $i=1$ and $n$ in the first constraint gives $|z^Tx-n|\leq \beta$ and $|z^Tx-n+2|\leq\beta$. It's necessary to have $\beta>1$ and $z^Tx=n-1$ to have both constraints strictly satisfiable. Since we can always find $z^\star$ satisfying the below linear system
\[\begin{cases}
x^Tz^\star=n-1\\
1^Tz^\star=0\\
|x-1x_i|^Tz^\star-1^T\signfcn(x-1x_i)=-1 \quad\forall i=2,\cdots,n-1
\end{cases}\]
Note such $z^\star$ also satisfies
\[||-x+1x_i|^Tz^\star+1^T\signfcn(-x+1x_i)|\leq 1\]
Therefore when $\beta>1,$ the original constraint set is strictly feasible.
\end{proof}

\begin{lemma}\label{lemma_3}
The below constraint set is strictly feasible only when $\beta>2 .$
\[\begin{cases}
|z^T|x-1x_i|-1^T\signfcn(x-1x_i)|\leq\beta\\
|z^T|-x+1x_i|+1^T\signfcn(-x+1x_i)|\leq\beta\\
z^T1=0\\
z^Tx=n
\end{cases} \forall i=1,\cdots,n\] 
\begin{proof}
Consider without loss of generality that $x_1<x_2<\cdots<x_n.$ Then taking $i=n$ in the first constraint gives $|-n+(n-2)|\leq \beta,$ which indicates that $\beta>2$ is necessary for the constraint set to be strictly feasible. Since we can always find $z^\star$ satisfying 
\[\begin{cases}
x^Tz^\star=n\\
1^Tz^\star=0\\
|x-1x_i|^Tz^\star-1^T\signfcn(x-1x_i)=0\quad \forall i=2,\cdots,n-1
\end{cases}\]
Note such $z^\star$ also satisfies
\[||-x+1x_i|^Tz^\star+1^T\signfcn(-x+1x_i)|\leq 2\]
Therefore when $\beta>2$, the original constraint set is strictly feasible.
\end{proof}
\end{lemma}
\subsection{Proof of Theorem \ref{thm3}}\label{thm31_proof}
\subsubsection{Proof of Theorem \ref{thm3} for ReLU activation}
\begin{proof}
Consider data $x\in\mathbb{R}^{n}$. Let $m$ denote number of hidden neurons, then we have first layer weight $w\in\mathbb{R}^m,$ first layer bias $b\in\mathbb{R}^m,$ second layer weight $\alpha\in\mathbb{R}^{m}$ and second layer bias $b_0\in\mathbb{R}.$  The score matching objective is reduced to
\begin{equation*}
p^\star=\min_{w,\alpha,b}\frac{1}{2}\left\|\sum_{j=1}^m (x w_j+1 b_j)_+\alpha_j+1 b_0\right\|_2^2+1^T\left(\sum_{j=1}^m w_j\alpha_j \indicfcn\{x w_j+1 b_j\geq 0\}\right)+\frac{1}{2}\beta \sum_{j=1}^m (w_j^2+\alpha_j^2).
\end{equation*}
According to Lemma 2 in \cite{pilanci2020neural}, after rescaling, the above problem is equivalent to 
\begin{equation*}
\min_{\substack{w,\alpha,b\\|w_j|=1}}\frac{1}{2}\left\|\sum_{j=1}^m (x w_j+1 b_j)_+\alpha_j+1 b_0\right\|_2^2+1^T\left(\sum_{j=1}^m w_j\alpha_j \indicfcn\{x w_j+1 b_j\geq 0\}\right)+\beta \sum_{j=1}^m |\alpha_j|,
\end{equation*}
which can be written as
\begin{equation*}
\begin{aligned}
& \min_{\substack{w,\alpha,b,r_1,r_2\\|w_j|=1}} \frac{1}{2}\|r_1\|_2^2+1^Tr_2+\beta\sum_{j=1}^m |\alpha_j|\\
&\quad \mbox{s.t. } r_1=\sum_{j=1}^m (x w_j+1 b_j)_+\alpha_j+1 b_0\\
&\qquad r_2 = \sum_{j=1}^m w_j\alpha_j\indicfcn\{x w_j+1 b_j\geq 0\}.
\end{aligned}
\end{equation*}
The dual problem writes
\begin{equation*}
\begin{aligned}
   &d^\star= \max_{z_1,z_2} \min_{\substack{w,\alpha,b,r_1,r_2\\|w_j|=1}}\frac{1}{2}\|r_1\|_2^2+1^Tr_2 +\beta\sum_{j=1}^m |\alpha_j|+z_1^T\left(r_1-\sum_{j=1}^m (x w_j+1 b_j)_+\alpha_j-1 b_0\right)\\
&\qquad \qquad\qquad \qquad\qquad \qquad\qquad\qquad \qquad\qquad  +z_2^T\left(r_2-\sum_{j=1}^m w_j\alpha_j\indicfcn\{x w_j+1 b_j\geq 0\}\right),
\end{aligned}
\end{equation*}
which gives a lower bound of $p^\star$. Minimizing  over $r_1,r_2,\alpha_j$ above gives
\begin{equation*}
\begin{aligned}
\max_{z} &\min_{b_0} -\frac{1}{2}\|z\|_2^2-b_0z^T1\\
& \mbox{s.t. } 
|z^T(x w_j+1 b_j)_+-w_j1^T \indicfcn\{x w_j+1 b_j\geq 0\}|\leq \beta,
\quad \forall |W_j|=1, \forall b_j.
\end{aligned}
\end{equation*}
For the constraints to hold, we must have $z^T1=0$ and $b_j$ takes values over $x_j$'s. The above is equivalent to
\begin{equation}\label{1d_dual_3}
\begin{aligned}
&\max_{z} \quad-\frac{1}{2}\|z\|_2^2\\
& \mbox{s.t.} 
\begin{cases}
|z^T(x-1 x_i)_+-1^T\indicfcn\{x-1 x_i\geq 0\}|\leq \beta \quad \\
|z^T(x-1 x_i)_+-1^T\indicfcn\{x-1 x_i>0\}|\leq \beta \quad \\
|z^T(-x+1 x_i)_++1^T\indicfcn\{-x+1x_i\geq 0\}|\leq \beta \quad \\
|z^T(-x+1 x_i)_++1^T\indicfcn\{-x+1 x_i>0\}|\leq \beta \quad \\
z^T1 = 0
\end{cases}   \forall i=1,\ldots,n,
\end{aligned}
\end{equation}
According to Lemma \ref{lemma_2}, when $\beta\geq 1,$ the constraints in (\ref{1d_dual_3}) are feasible for affine constraints, thus Slater's condition holds and the dual problem  writes
\begin{equation*}
\begin{aligned}
&d^\star=\min_{\substack{z_0,\ldots,z_7,z_8\\ \mbox{s.t.} z_0,\ldots,z_7\geq 0}} \max_{z} -\frac{1}{2}\|z\|_2^2  + \sum_{i=1}^n z_{0i}\left(z^T(x-1x_i)_+-1^T\indicfcn\{x-1x_i\geq 0\}+\beta\right)\\
&+ \sum_{i=1}^n z_{1i}\left(-z^T(x-1x_i)_++1^T\indicfcn\{x-1x_i\geq 0\}+\beta\right)+ \sum_{i=1}^n z_{2i}\left(z^T(x-1x_i)_+-1^T\indicfcn\{x-1x_i>0\}+\beta\right)\\
&+ \sum_{i=1}^n z_{3i}\left(-z^T(x-1x_i)_++1^T\indicfcn\{x-1x_i>0\}+\beta\right)+ \sum_{i=1}^n z_{4i}\left(z^T(-x+1x_i)_++1^T\indicfcn\{-x+1x_i\geq 0\}+\beta\right)\\
& + \sum_{i=1}^n z_{5i}\left(-z^T(-x+1x_i)_+-1^T\indicfcn\{-x+1x_i\geq 0\}+\beta\right)+ \sum_{i=1}^n z_{6i}\left(z^T(-x+1x_i)_++1^T\indicfcn\{-x+1x_i>0\}+\beta\right)\\
& + \sum_{i=1}^n z_{7i}\left(-z^T(-x+1x_i)_+-1^T\indicfcn\{-x+1x_i>0\}+\beta\right)
+z_8z^T1,
\end{aligned}
\end{equation*}
which is equivalent to
\begin{equation*}
\begin{aligned}
& \min_{\substack{z_0,\ldots,z_7 ,z_8\\ \mbox{s.t.} z_0,\ldots,z_7\geq 0}} \max_{z} -\frac{1}{2}\|z\|_2^2+e^Tz+f,
\end{aligned}
\end{equation*}
where 
\begin{equation*}
\begin{aligned}
e =\sum_{i=1}^n z_{0i}(x-1x_i)_+-\sum_{i=1}^n z_{1i}(x-1x_i)_++\sum_{i=1}^n z_{2i}(x-1x_i)_+-\sum_{i=1}^n z_{3i}(x-1x_i)_+ +\sum_{i=1}^n z_{4i}(-x+1x_i)_+\\-\sum_{i=1}^n z_{5i}(-x+1x_i)_++\sum_{i=1}^n z_{6i}(-x+1x_i)_+-\sum_{i=1}^n z_{7i}(-x+1x_i)_++1z_8,
\end{aligned}
\end{equation*}
and 
\begin{equation*}
\begin{aligned}
f=-\sum_{i=1}^n z_{0i}1^T\indicfcn\{x-1x_i\geq 0\}+\sum_{i=1}^n z_{1i}1^T\indicfcn\{x-1x_i\geq 0\}-\sum_{i=1}^n z_{2i}1^T\indicfcn\{x-1x_i>0\}\\+\sum_{i=1}^n z_{3i}1^T\indicfcn\{x-1x_i>0\}+\sum_{i=1}^n z_{4i}1^T\indicfcn\{-x+1 x_i\geq 0\}-\sum_{i=1}^n z_{5i}1^T\indicfcn\{-x+1x_i\geq 0\}\\+\sum_{i=1}^n z_{6i}1^T\indicfcn\{-x+1x_i> 0\}-\sum_{i=1}^n z_{7i}1^T\indicfcn\{-x+1x_i> 0\}+\beta(\sum_{i=0}^7\|z_i\|_1).
\end{aligned}
\end{equation*}
Maximizing over $z$ gives
\begin{equation*}
\begin{aligned}
\min_{\substack{z_0,\ldots,z_7,z_8\\ \mbox{s.t.} z_0,\ldots,z_7\geq 0}} \frac{1}{2}\|e\|_2^2 + f,
\end{aligned}
\end{equation*}
Simplifying to get
\begin{equation*}
\begin{aligned}
\min_{y_0,y_1, y_2,y_3,y_4} \frac{1}{2}\left\|A_1(y_0+y_1)+A_2(y_2+y_3)+1 y_4\right\|_2^2+1^TC_1y_0-1^TC_3y_2\\
+1^TC_2y_1-1^TC_4y_3+\beta(\|y_0\|_1+\|y_1\|_1+\|y_2\|_1+\|y_3\|_1).
\end{aligned}
\end{equation*}
Minimizing over $y_4$ gives the convex program (\ref{relu_noskip_formula}) in Theorem \ref{thm3} with $A=[\bar A_1,\bar A_1,\bar A_2,\bar A_2]\in \mathbb{R}^{n\times 4n}, b=[1^TC_1,1^TC_2,-1^TC_3,-1^TC_4]^T\in\mathbb{R}^{4n}$ where $\bar A_1=\left(I-\frac{1}{n}11^T\right)A_1, \bar A_2=\left(I-\frac{1}{n}11^T\right)A_2$ with $[A_1]_{ij}=(x_i-x_j)_+$ and $[A_2]_{ij}=(-x_i+x_j)_+, [C_1]_{ij}=\indicfcn\{x_i-x_j\geq 0\},[C_2]_{ij}=\indicfcn\{x_i-x_j> 0\},[C_3]_{ij}=\indicfcn\{-x_i+x_j\geq 0\},[C_4]_{ij}=\indicfcn\{-x_i+x_j> 0\}$. Once we obtain optimal solution $y^\star$ to problem (\ref{relu_noskip_formula}), we can take 
\begin{equation}\label{for_b0}
\begin{cases}
w_j^\star=\sqrt{|y^\star_j|}, \alpha_j^\star=\sqrt{|y^\star_j|},b^\star_j=-\sqrt{|y^\star_j|}x_j \text{ for }j=1,\ldots,n,\\
w_j^\star=\sqrt{|y^\star_j|}, \alpha_j^\star=\sqrt{|y^\star_j|},b^\star_j=-\sqrt{|y^\star_j|}(x_{j-n}+\epsilon) \text{ for }j=n+1,\ldots,2n,\\
w_j^\star=-\sqrt{|y^\star_j|}, \alpha_j^\star=\sqrt{|y^\star_j|},b^\star_j=\sqrt{|y^\star_j|}x_{j-2n}\text{ for }j=2n+1,\ldots,3n,\\
w_j^\star=-\sqrt{|y^\star_j|}, \alpha_j^\star=\sqrt{|y^\star_j|},b^\star_j=\sqrt{|y^\star_j|}(x_{j-3n}-\epsilon)\text{ for }j=3n+1,\ldots,4n,\\
b_0^\star=-\frac{1}{n}1^T([A_1,A_1,A_2,A_2]y^\star),
\end{cases}
\end{equation}
then score matching objective has the same value as optimal value of convex program (\ref{relu_noskip_formula}) as $\epsilon\rightarrow 0$, which indicates $p^\star=d^\star$  and the above parameter set is optimal.
\end{proof}
\subsubsection{Proof of Theorem \ref{thm3} for Absolute Value Activation}\label{abs_noskip_proof}
\begin{proof}
Consider data $x\in\mathbb{R}^{n}$. Let $m$ denote number of hidden neurons, then we have first layer weight $w\in\mathbb{R}^m,$ first layer bias $b\in\mathbb{R}^m,$ second layer weight $\alpha\in\mathbb{R}^{m}$ and second layer bias $b_0\in\mathbb{R}.$  Then the score matching objective is reduced to
\begin{equation*}
p^\star=\min_{w,\alpha,b}\frac{1}{2}\left\|\sum_{j=1}^m |x w_j+1 b_j|\alpha_j+1 b_0\right\|_2^2+1^T\left(\sum_{j=1}^m w_j\alpha_j \signfcn(x w_j+1 b_j)\right)+\frac{1}{2}\beta \sum_{j=1}^m (w_j^2+\alpha_j^2).
\end{equation*}
According to Lemma 2 in \cite{pilanci2020neural}, after rescaling, the above problem is equivalent to 
\[\min_{\substack{w,\alpha,b\\|w_j|=1}} \frac{1}{2}\left\|\sum_{j=1}^m |x w_j+1 b_j|\alpha_j+1 b_0\right\|_2^2+1^T\left(\sum_{j=1}^m w_j\alpha_j\signfcn(x w_j+1 b_j)\right)+\beta \sum_{j=1}^m |\alpha_j|,\]
which can be written as
\begin{equation}\label{1d_1}
\begin{aligned}
& \min_{\substack{w,\alpha,b,r_1,r_2\\|w_j|=1}} &\frac{1}{2}\|r_1\|_2^2+1^Tr_2+\beta\sum_{j=1}^m |\alpha_j|\\
&\qquad \mbox{s.t. } &r_1=\sum_{j=1}^m |x w_j+1 b_j|\alpha_j+1 b_0\\
&\qquad\qquad &r_2 = \sum_{j=1}^m w_j\alpha_j\signfcn(x w_j+1 b_j).
\end{aligned}
\end{equation}
The dual problem of (\ref{1d_1}) writes
\begin{equation*}
\begin{aligned}
   &d^\star= \max_{z_1,z_2} \min_{\substack{w,\alpha,b,r_1,r_2\\|w_j|=1}}\frac{1}{2}\|r_1\|_2^2+1^Tr_2 +\beta\sum_{j=1}^m |\alpha_j|+z_1^T\left(r_1-\sum_{j=1}^m |x w_j+1 b_j|\alpha_j-1 b_0\right)\\
&\qquad \qquad\qquad \qquad\qquad \qquad\qquad\qquad \qquad\qquad  +z_2^T\left(r_2-\sum_{j=1}^m w_j\alpha_j\signfcn(x w_j+1 b_j)\right),
\end{aligned}
\end{equation*}
which is a lower bound of optimal value to the original problem, i.e., $p^\star\geq d^\star$.  Minimizing over $r_1$ and $r_2$ gives 
\[\max_{z} \min_{\substack{w,\alpha,b\\|w_j|=1}} -\frac{1}{2}\|z\|_2^2 +\beta\sum_{j=1}^m |\alpha_j|-z^T\left(\sum_{j=1}^m|x w_j+1 b_j|\alpha_j+1 b_0\right)+1^T\sum_{j=1}^m w_j\alpha_j\signfcn(x w_j+1 b_j).\]
Minimizing over $\alpha_j$ gives
\begin{equation*}
\begin{aligned}
\max_{z} &\min_{b_0} -\frac{1}{2}\|z\|_2^2-b_0z^T1\\
& \mbox{s.t. } 
|z^T|x w_j+1 b_j|-w_j1^T \signfcn(x w_j+1 b_j)|\leq \beta,
\quad \forall |w_j|=1, \forall b_j,
\end{aligned}
\end{equation*}
which is equivalent to 
\begin{equation*}
\begin{aligned}
\max_{z} &\min_{b_0} -\frac{1}{2}\|z\|_2^2-b_0z^T1\\
&\mbox{s.t.} 
\begin{cases}
|z^T|x+1 b_j|-1^T\signfcn(x+1 b_j)|\leq \beta\\
|z^T|-x+1 b_j|+1^T\signfcn(-x+1 b_j)|\leq\beta
\end{cases} \forall b_j.
\end{aligned}
\end{equation*}
For the constraints to hold, we must have $z^T1=0$ and $b_j$ takes values over $x_j$'s. Furthermore, since $\signfcn$ is discontinuous at input 0, we add another function $\signfcn^\ast$ which takes value $-1$ at input 0 to cater for the constraints. The above is equivalent to
\begin{equation}\label{1d_dual}
\begin{aligned}
&\max_{z} \quad-\frac{1}{2}\|z\|_2^2\\
& \mbox{s.t.} 
\begin{cases}
|z^T|x-1 x_i|-1^T\signfcn(x-1 x_i)|\leq \beta \quad \\ 
|z^T|x-1 x_i|-1^T\signfcn^\ast(x-1 x_i)|\leq \beta \quad \\
|z^T|-x+1 x_i|+1^T\signfcn(-x+1 x_i)|\leq \beta \quad \\
|z^T|-x+1 x_i|+1^T\signfcn^\ast(-x+1 x_i)|\leq \beta \quad \\
z^T1 = 0
\end{cases}   \forall i=1,\ldots,n.
\end{aligned}
\end{equation}
Since the second constraint overlaps with the third, and the fourth constraint overlaps with the first, (\ref{1d_dual}) is equivalent to 
\begin{equation}\label{1d_dual2}
\begin{aligned}
&\max_{z} -\frac{1}{2}\|z\|_2^2\\
& \mbox{s.t.} 
\begin{cases}
|z^T|x-1 x_i|-1^T\signfcn(x-1 x_i)|\leq \beta \quad\\ 
|z^T|-x+1 x_i|+1^T\signfcn(-x+1 x_i)|\leq \beta \quad \\
z^T1 = 0
\end{cases}  \forall i=1,\ldots,n.
\end{aligned}
\end{equation}
According to Lemma \ref{lemma_1}, when $\beta\geq 1,$ the constraints in (\ref{1d_dual2}) are 
 feasible for affine constraints, thus Slater's condition holds and the dual problem  writes 
\begin{equation*}
\begin{aligned}
&d^\star=\min_{\substack{z_0,z_1,z_2,z_3,z_4\\ \mbox{s.t.} z_0,z_1,z_2,z_3\geq 0}} \max_{z} -\frac{1}{2}\|z\|_2^2  + \sum_{i=1}^n z_{0i}\left(z^T|x-1x_i|-1^T\signfcn(x-1 x_i)+\beta\right)\\
&\qquad \qquad\qquad\qquad\qquad + \sum_{i=1}^n z_{1i}\left(-z^T|x-1 x_i|+1^T\signfcn(x-1 x_i)+\beta\right)\\
&\qquad \qquad\qquad\qquad\qquad + \sum_{i=1}^n z_{2i}\left(z^T|-x+1 x_i|+1^T\signfcn(-x+1x_i)+\beta\right)\\
&\qquad \qquad\qquad\qquad\qquad + \sum_{i=1}^n z_{3i}\left(-z^T|-x+1x_i|-1^T\signfcn(-x+1x_i)+\beta\right)\\
& \qquad \qquad\qquad\qquad\qquad +z_4z^T1,
\end{aligned}
\end{equation*}
which is equivalent to
\begin{equation*}
\begin{aligned}
& \min_{\substack{z_0,z_1,z_2,z_3,z_4\\ \mbox{s.t.} z_0,z_1,z_2,z_3\geq 0}} \max_{z} -\frac{1}{2}\|z\|_2^2+e^Tz+f,
\end{aligned}
\end{equation*}
where 
\[e =\sum_{i=1}^n z_{0i}|x-1x_i|-\sum_{i=1}^n z_{1i}|x-1x_i|+\sum_{i=1}^n z_{2i}|-x+1x_i|-\sum_{i=1}^n z_{3i}|-x+1x_i|+1z_4\]
and
\begin{equation*}
\begin{aligned}
f=-\sum_{i=1}^n z_{0i}1^T\signfcn(x-1x_i)+\sum_{i=1}^n z_{1i}1^T\signfcn(x-1x_i)+\sum_{i=1}^n z_{2i}1^T\signfcn(-x+1x_i)\\-\sum_{i=1}^n z_{3i}1^T\signfcn(-x+1x_i)+\beta(\|z_0\|_1+\|z_1\|_1+\|z_2\|_1+\|z_3\|_1).
\end{aligned}
\end{equation*}
Maximizing over $z$ gives
\begin{equation*}
\begin{aligned}
\min_{\substack{z_0,z_1,z_2,z_3,z_4\\ \mbox{s.t.} z_0,z_1,z_2,z_3\geq 0}} \frac{1}{2}\|e\|_2^2 + f.
\end{aligned}
\end{equation*}
Simplifying to get
\begin{equation*}
\begin{aligned}
\min_{y_1,y_2, z} \frac{1}{2}\left\|A_1(y_1+y_2)+1 z\right\|_2^2+1^TC_1y_1-1^TC_2y_2+\beta(\|y_1\|_1+\|y_2\|_1).
\end{aligned}
\end{equation*}
Minimizing over $z$ gives the convex program (\ref{relu_noskip_formula}) in Theorem \ref{thm3} where $A=[ \bar A_1,\bar A_1]\in \mathbb{R}^{n\times 2n}, b=[1^TC_1,-1^TC_2]^T\in\mathbb{R}^{2n}$ with $\bar A_1=\left(I-\frac{1}{n}11^T\right)A_1$, $[A_1]_{ij}=|x_i- x_j|$, $[C_1]_{ij}=\signfcn(x_i-x_j)$ and $[C_2]_{ij}=\signfcn(-x_i+x_j)$. Once we obtain optimal solution $y^\star$ to problem (\ref{relu_noskip_formula}), we can take 
\[\begin{cases}
w_j^\star=\sqrt{|y^\star_j|}, \alpha_j^\star=\sqrt{|y^\star_j|},b^\star_j=-\sqrt{|y^\star_j|}x_j \text{ for }j=1,\ldots,n,\\
w_j^\star=-\sqrt{|y^\star_j|}, \alpha_j^\star=\sqrt{|y^\star_j|},b^\star_j=\sqrt{|y^\star_j|}x_{j-n}\text{ for }j=n+1,\ldots,2n,\\
b_0^\star=-\frac{1}{n}1^T([A_1,A_1]y^\star),
\end{cases}\]
then score matching objective  has the same value as optimal value of convex program (\ref{relu_noskip_formula}), which indicates $p^\star=d^\star$  and the above parameter set is optimal.
\end{proof}
\subsection{Convex Programs for  More Model Architectures}\label{sec3_more_archi}
\subsubsection{ReLU Activation with Skip Connection}\label{relu_skip_section}
\begin{theorem}\label{relu_skip_thm}
\textcolor{black}{When $\sigma$ is  ReLU } and $V\neq 0$ , denote the optimal score matching objective value (\ref{train_obj}) with $s_\theta$ specified in (\ref{general_arc})  as $p^\star,$ when $m\geq \text{len}(y)$ and $\beta\geq 1,$  
\begin{equation}\label{relu_skip_formula}
\begin{aligned}
p^*=& \min_{y}\quad \frac{1}{2}\|Ay\|_2^2+b^Ty+c+2\beta\|y\|_1\,,
\end{aligned}
\end{equation}
 $A,b,c$ and reconstruction rule for $\theta$ is specficied in the proof below.
\end{theorem}

\begin{proof}
Here we reduce score matching objective including ReLU activation to score matching objective including absolute value activation and exploits results in Theorem \ref{abs_skip_thm}. Let $\{w^r,b^r,\alpha^r,v^r\}$ denotes parameter set corresponding to ReLU activation, consider another parameter set 
$\{w^a,b^a,\alpha^a,v^a\}$ satisfying 
\[\begin{cases}
\alpha^r=2\alpha^a,\\
w^r=w^a,\\
b^r=b^a,\\
b_0^r=b_0^a-\frac{1}{2}\sum_{j=1}^m b_j^r\alpha_j^r,\\
v^r=v^a-\frac{1}{2}\sum_{j=1}^m w_j^r\alpha_j^r.
\end{cases}\]
Then the score matching objective 
\begin{equation*}
\min_{w^r,\alpha^r,b^r,v^r}\frac{1}{2}\left\|\sum_{j=1}^m (x w^r_j+1 b^r_j)_+\alpha^r_j+x v^r+1 b_0^r\right\|_2^2+1^T\left(\sum_{j=1}^m w^r_j\alpha^r_j \indicfcn\{x w^r_j+1b^r_j\geq 0\}\right)+nv^r+\frac{\beta}{2} \sum_{j=1}^m ({w^r_j}^2+{\alpha^r_j}^2)
\end{equation*}
is equivalent to
\[
\min_{w^a,\alpha^a,b^a,v^a}\frac{1}{2}\left\|\sum_{j=1}^m |x w^a_j+1 b^a_j|\alpha^a_j+x v^a+1 b_0^a\right\|_2^2+1^T\left(\sum_{j=1}^m w^a_j\alpha^a_j\signfcn(x w^a_j+1 b^a_j)\right)+nv^a+\frac{\beta}{2}\sum_{j=1}^m \left({w^a_j}^2+4{\alpha^a_j}^2\right).
\]
According to Lemma 2 in \cite{pilanci2020neural}, after rescaling, the above problem is equivalent to 
\begin{equation}\label{relu_abs_eq}
\min_{\substack{w^a,\alpha^a,b^a,v^a\\|w^a_j|=1}}\frac{1}{2}\left\|\sum_{j=1}^m |x w^a_j+1 b^a_j|\alpha^a_j+x v^a+1 b_0^a\right\|_2^2+1^T\left(\sum_{j=1}^m w^a_j\alpha^a_j\signfcn(x w^a_j+1 b^a_j)\right)+nv^a+2\beta\sum_{j=1}^m |\alpha^a_j|.
\end{equation}
Following similar analysis as in  Appendix \ref{abs_skip_thm} with a different rescaling factor we can derive the convex program (\ref{relu_skip_formula}) with $A=B^{\frac{1}{2}}A_1,b=A_1^T(-n\bar x/\|\bar x\|_2^2)+b_1,c=-n^2/(2\|\bar x\|_2^2)$ where $B=I-P_{\bar x}$ with $P_{\bar x}=\bar x\bar x^T/\|\bar x\|_2^2$, and $A_1, b_1$ are identical to $A, b$ defined in Section \ref{abs_noskip_proof} respectively. Here, $\bar x_j:=x_j-\sum_i x_i/n$ denotes mean-subtracted data vector.  The optimal solution set to (\ref{relu_abs_eq}) is given by 
\[\begin{cases}
{w^a_j}^\star=\sqrt{2|y^\star_j|}, {\alpha^a_j}^\star=\sqrt{|y^\star_j|/2},{b^a_j}^\star=-\sqrt{2|y^\star_j|}x_j \text{ for }j=1,\ldots,n,\\
{W^a_j}^\star=-\sqrt{2|y^\star_j|}, {\alpha^a_j}^\star=\sqrt{|y^\star_j|/2},{b^a_j}^\star=\sqrt{2|y^\star_j|}x_{j-n}\text{ for }j=n+1,\ldots,2n,\\
{v^a}^\star=-(\bar x^TA_1y^\star+n)/\|\bar x\|_2^2,\\
{b_0^a}^\star=-\frac{1}{n}1^T([A_1',A_1']y^\star+x {v^a}^\star),
\end{cases}\]
where $A_1'$ is $A_1$ defined in Appendix \ref{abs_noskip_proof} and $y^\star$ is optimal solution to convex program (\ref{relu_skip_formula}). Then the optimal parameter set $\{w^r,b^r,\alpha^r,z^r\}$ is given by
\[\begin{cases}
{w^r_j}^\star=\sqrt{2|y^\star_j|}, {\alpha^r_j}^\star=\sqrt{2|y^\star_j|},{b^r_j}^\star=-\sqrt{2|y^\star_j|}x_j \text{ for }j=1,\ldots,n,\\
{w^r_j}^\star=-\sqrt{2|y^\star_j|}, {\alpha^r_j}^\star=\sqrt{2|y^\star_j|},{b^r_j}^\star=\sqrt{2|y^\star_j|}x_{j-n}\text{ for }j=n+1,\ldots,2n,\\
{v^r}^\star=-(\bar x^TA_1y^\star+n)/\|\bar x\|_2^2-\sum_{j=1}^m {w^r_j}^\star {\alpha^r_j}^\star/2,\\
{b_0^r}^\star=-\frac{1}{n}1^T([A_1',A_1']y^\star+x (-(\bar x^TA_1y^\star+n)/\|\bar x\|_2^2))-\sum_{j=1}^m {b^r_j}^\star{\alpha^r_j}^\star/2.
\end{cases}\]

\end{proof}
\subsubsection{Absolute Value Activation with Skip Connection}\label{abs_skip_sec}
\begin{theorem}\label{abs_skip_thm}
\textcolor{black}{When $\sigma$ is  absolute value activation } and $V\neq 0$ , denote the optimal score matching objective value (\ref{train_obj}) with $s_\theta$ specified in (\ref{general_arc})  as $p^\star,$ when $m\geq \text{len}(y)$ and $\beta\geq 2,$  
\begin{equation}\label{abs_skip_formula}
\begin{aligned}
p^*=& \min_{y}\quad \frac{1}{2}\|Ay\|_2^2+b^Ty+\beta\|y\|_1\,,
\end{aligned}
\end{equation}
 $A,b$ and reconstruction rule for $\theta$ is specficied in the proof below.
\end{theorem}
\begin{proof}
Consider data matrix $x\in\mathbb{R}^{n}$,  then the score matching objective is reduced to
\begin{equation*}
p^\star=\min_{w,\alpha,b,v}\frac{1}{2}\left\|\sum_{j=1}^m |x w_j+1 b_j|\alpha_j+x v+1 b_0\right\|_2^2+1^T\left(\sum_{j=1}^m w_j\alpha_j \signfcn(x w_j+1 b_j)\right)+nv+\frac{1}{2}\beta \sum_{j=1}^m (w_j^2+\alpha_j^2).
\end{equation*}
Following similar analysis as in Appendix \ref{abs_noskip_proof}, we can derive the dual problem as
\begin{equation*}
\begin{aligned}
   &d^\star= \max_{z_1,z_2} \min_{\substack{w,\alpha,b,v,r_1,r_2\\|w_j|=1}}\frac{1}{2}\|r_1\|_2^2+1^Tr_2 +nv+\beta\sum_{j=1}^m |\alpha_j|+z_1^T\left(r_1-\sum_{j=1}^m |x w_j+1 b_j|\alpha_j-x v-1b_0\right)\\
&\qquad \qquad \qquad\qquad \qquad\qquad \qquad\qquad\qquad \qquad\qquad\qquad  +z_2^T\left(r_2-\sum_{j=1}^m w_j\alpha_j\signfcn(x w_j+1 b_j)\right).
\end{aligned}
\end{equation*}
which gives a lower bound of $p^\star$. Minimizing over $r_1$ and $r_2$ gives 
\[\max_{z_1} \min_{\substack{w,\alpha,b,v\\|w_j|=1}} -\frac{1}{2}\|z_1\|_2^2 +nv+\beta\sum_{j=1}^m |\alpha_j|-z_1^T\left(\sum_{j=1}^m|x w_j+1 b_j|\alpha_j+x v+1 b_0\right)+1^T\sum_{j=1}^m w_j\alpha_j\signfcn(x w_j+1 b_j).\]
Minimizing over $v$ gives
\[\max_{\substack{z_1\\z_1^Tx=n}} \min_{\substack{w,\alpha,b\\|w_j|=1}} -\frac{1}{2}\|z_1\|_2^2 +\beta\sum_{j=1}^m |\alpha_j|-z_1^T\left(\sum_{j=1}^m|x w_j+1 b_j|\alpha_j+1b_0\right)+1^T\sum_{j=1}^m w_j\alpha_j\signfcn(x w_j+1 b_j).\]
Minimizing over $\alpha_j$ gives
\begin{equation*}
\begin{aligned}
\max_{z} &\min_{b_0} -\frac{1}{2}\|z\|_2^2-b_0z^T1\\
& \mbox{s.t. } 
\begin{cases}
z^Tx=n\\
|z^T|x w_j+1 b_j|-w_j1^T \signfcn(x w_j+1 b_j)|\leq \beta,
\quad \forall |w_j|=1, \forall b_j.
\end{cases}
\end{aligned}
\end{equation*}
Following same logic as in Appendix \ref{abs_noskip_proof}, the above problem is equivalent to
\begin{equation}\label{1d_dual_2}
\begin{aligned}
&\max_{z} \quad-\frac{1}{2}\|z\|_2^2\\
& \mbox{s.t.} 
\begin{cases}
|z^T|x-1 x_i|-1^T\signfcn(x-1 x_i)|\leq \beta \quad \\ 
|z^T|-x+1 x_i|+1^T\signfcn(-x+1 x_i)|\leq \beta \quad \\
z^T1 = 0\\
z^Tx=n
\end{cases}   \forall i=1,\ldots,n.
\end{aligned}
\end{equation}
According to Lemma \ref{lemma_3}, when $\beta\geq 2,$ the constraints in (\ref{1d_dual_2}) are   feasible for affine constraints, thus Slater's condition holds and the dual problem  writes
\begin{equation*}
\begin{aligned}
&d^\star=\min_{\substack{z_0,z_1,z_2,z_3,z_4,z_5\\ \mbox{s.t.} z_0,z_1,z_2,z_3\geq 0}} \max_{z} -\frac{1}{2}\|z\|_2^2  + \sum_{i=1}^n z_{0i}\left(z^T|x-1x_i|-1^T\signfcn(x-1x_i)+\beta\right)\\
&\qquad \qquad\qquad\qquad\qquad + \sum_{i=1}^n z_{1i}\left(-z^T|x-1x_i|+1^T\signfcn(x-1x_i)+\beta\right)\\
&\qquad \qquad\qquad\qquad\qquad + \sum_{i=1}^n z_{2i}\left(z^T|-x+1x_i|+1^T\signfcn(-x+1x_i)+\beta\right)\\
&\qquad \qquad\qquad\qquad\qquad + \sum_{i=1}^n z_{3i}\left(-z^T|-x+1x_i|-1^T\signfcn(-x+1x_i)+\beta\right)\\
& \qquad \qquad\qquad\qquad\qquad +z_4(z^Tx-n)+z_5z^T1,
\end{aligned}
\end{equation*}
which is equivalent to
\begin{equation*}
\begin{aligned}
& \min_{\substack{z_0,z_1,z_2,z_3,z_4, z_5\\ \mbox{s.t.} z_0,z_1,z_2,z_3\geq 0}} \max_{z} -\frac{1}{2}\|z\|_2^2+e^Tz+f,
\end{aligned}
\end{equation*}
where
\[e =\sum_{i=1}^n z_{0i}|x-1x_i|-\sum_{i=1}^n z_{1i}|x-1x_i|+\sum_{i=1}^n z_{2i}|-x+1x_i|-\sum_{i=1}^n z_{3i}|-x+1x_i|+xz_4+1z_5,\]
and
\begin{equation*}
\begin{aligned}
f=-\sum_{i=1}^n z_{0i}1^T\signfcn(x-1x_i)+\sum_{i=1}^n z_{1i}1^T\signfcn(x-1x_i)+\sum_{i=1}^n z_{2i}1^T\signfcn(-x+1x_i)\\-\sum_{i=1}^n z_{3i}1^T\signfcn(-x+1x_i)-z_4n+\beta(\|z_0\|_1+\|z_1\|_1+\|z_2\|_1+\|z_3\|_1).
\end{aligned}
\end{equation*}
Maximizing over $z$ gives
\begin{equation*}
\begin{aligned}
\min_{\substack{z_0,z_1,z_2,z_3,z_4,z_5\\ \mbox{s.t.} z_0,z_1,z_2,z_3\geq 0}} \frac{1}{2}\|e\|_2^2 + f.
\end{aligned}
\end{equation*}
Simplifying to get
\begin{equation}\label{final_obj_2}
\begin{aligned}
\min_{y_0,y_1, y_2,y_3} \frac{1}{2}\left\|A_1'(y_0+y_1)+x y_2+1 y_3\right\|_2^2+1^TC_1y_0-1^TC_2y_1+ny_2+\beta(\|y_1\|_1+\|y_2\|_1),
\end{aligned}
\end{equation}
where $A_1',C_1,C_2$ are as $A_1,C_1,C_2$ defined in Appendix \ref{abs_noskip_proof}. Minimizing over $y_3$ gives $y_3=-1^T(A_1'(y_0+y_1)+x y_2)/n$ and (\ref{final_obj_2}) is reduced to
\[\min_{y_0,y_1, y_2} \frac{1}{2}\left\|\bar A_1'(y_0+y_1)+\bar xy_2\right\|_2^2+1^TC_1y_0-1^TC_2y_1+ny_2+\beta(\|y_1\|_1+\|y_2\|_1),\]
 where $\bar A_1'$ is as $\bar A_1$ defined in Appendix \ref{abs_noskip_proof}. Minimizing over $y_2$ gives $y_2=-\left(\bar x^T\bar A_1'(y_0+y_1)+n\right)/\|\bar x\|_2^2$ and the above problem is equivalent to the convex program (\ref{abs_skip_formula}) in Theorem \ref{abs_skip_thm} with $A=B^{\frac{1}{2}}A_1,b=A_1^T(-n\bar x/\|\bar x\|_2^2)+b_1,c=-n^2/(2\|\bar x\|_2^2)$ where $B=I-P_{\bar x}$ with $P_{\bar x}=\bar x\bar x^T/\|\bar x\|_2^2$, and $A_1, b_1$ are identical to $A, b$ defined in Section \ref{abs_noskip_proof} respectively. Once we obtain optimal solution $y^\star$ to problem (\ref{abs_skip_formula}), we can take 
 \[\begin{cases}
w_j^\star=\sqrt{|y^\star_j|}, \alpha_j^\star=\sqrt{|y^\star_j|},b^\star_j=-\sqrt{|y^\star_j|}x_j \text{ for }j=1,\ldots,n,\\
w_j^\star=-\sqrt{|y^\star_j|}, \alpha_j^\star=\sqrt{|y^\star_j|},b^\star_j=\sqrt{|y^\star_j|}x_{j-n}\text{ for }j=n+1,\ldots,2n,\\
v^\star=-(\bar x^TA_1y^\star+n)/\|\bar x\|_2^2,\\
b_0^\star=-\frac{1}{n}1^T([A_1',A_1']y^\star+x v^\star),
\end{cases}\]
then score matching objective  has the same value as optimal value of convex program (\ref{abs_skip_formula}), which indicates $p^\star=d^\star$ and the above parameter 
 set is optimal.

\end{proof}
\subsection{Proof of Theorem \ref{sm_nd_thm}}\label{thm32_proof}
Here we first depict the assumption required for Theorem \ref{sm_nd_thm} to hold in Assumption \ref{multi_assum}. Note if Assumption \ref{multi_assum} is not true, original Theorem \ref{sm_nd_thm} still holds with equal sign replaced by greater than or equal to, which can be trivially seen from our proof of Theorem \ref{sm_nd_thm} below. Assumption \ref{multi_assum} has already been characterized in Proposition 3.1 in \cite{mishkin2022fast}, here we restate it for sake of completeness. First we define for each activation pattern $D_i\in\mathcal{D},$ the set of vectors that induce $D_i$ as
\[\mathcal{K}_i=\{u\in\mathbb{R}^d:(2D_i-I)Xu\succeq 0\}.\]

\begin{assumption}\label{multi_assum}
Let $\mathcal{D}_X$ denote the activation pattern set $\mathcal{D}$ induced by dataset $X$, assume for any $D_i\in\mathcal{D}_X,$  $\mathcal{K}_i-\mathcal{K}_i=\mathbb{R}^d.$
\end{assumption}
According to Proposition 3.1 in \cite{mishkin2022fast}, Assumption \ref{multi_assum} is satisfied whenever data matrix $X$ is full row-rank. Empirically,  Assumption \ref{multi_assum} holds with high probability according to experiments in \cite{mishkin2022fast}.
\subsubsection{Formal Proof}\label{sm_nd_proof}
\begin{proof}
When $X\in\mathbb{R}^{n\times d}$ for some $d>1$. Let $m$ denote the number of hidden neurons, then the score matching objective can be reduced to
\begin{equation}\label{high-dim-train0}
p^\star=\min_{u_j,v_j} \sum_{i=1}^n \left(\frac{1}{2}\left\|\sum_{j=1}^m(X_iu_j)_+v_j^T\right\|_2^2+\text{tr}\left(\nabla_{X_i}\left[\sum_{j=1}^m (X_iu_j)_+v_j^T\right]\right)\right),
\end{equation}
which can be rewritten as 
\begin{equation}\label{high-dm-prob0}
\min_{u_j,v_j} \frac{1}{2}\left\|\sum_{j=1}^m (Xu_j)_+v_j^T\right\|_F^2+1^T\left(\sum_{j=1}^m \indicfcn\{Xu_j\geq 0\}v_j^Tu_j\right).
\end{equation}
Let $D'_j=\text{diag}\left(\mathbbm{1}\{Xu_j\geq 0\}\right)$, then problem (\ref{high-dm-prob0}) is equivalent to
\begin{equation}\label{high-dim-prob2}
\min_{u_j,v_j} \frac{1}{2}\left\|\sum_{j=1}^m D'_jXu_jv_j^T\right\|_F^2+\sum_{j=1}^m \text{tr}(D'_j)v_j^Tu_j.
\end{equation}
Thus 
\begin{align}
p^\star & =\min_{\substack{W_j=u_jv_j^T\\(2D'_j-I)Xu_j\geq 0}}\frac{1}{2}\left\|\sum_{j=1}^m D'_jXW_j\right\|_F^2+\sum_{j=1}^m \text{tr}(D'_j)\text{tr}(W_j) \label{const}\\
& \geq \min_{W_j} \frac{1}{2} \left\|\sum_{j=1}^P D_jXW_j\right\|_F^2+\sum_{j=1}^P \text{tr}(D_j)\text{tr}(W_j),  \label{high-dim-cvx}
\end{align}
where $D_1,\ldots,D_P$ enumerates all possible sign patterns of $\text{diag}\left(\mathbbm{1}\{Xu\geq 0\}\right).$ To prove the reverse direction, let $\{W_j^\star\}$ be the optimal solution to the convex program (\ref{high_dim_cvx0}), we provide a way to reconstruct optimal $\{u_j,v_j\}$ which achieves the lower bound value. We first factorize each  $W_j^\star=\sum_{k=1}^d \tilde u_{jk}\tilde v_{jk}^T$. According to Theorem 3.3 in \citep{mishkin2022fast}, for any $\{j,k\}$, under Assumption \ref{multi_assum}, we can write $\tilde u_{jk}=\tilde u_{jk}'-\tilde u_{jk}''$ such that $\tilde u_{jk}',\tilde u_{jk}''\in \mathcal{K}_j$ with $\mathcal{K}_j=\{u\in\mathbb{R}^d:(2D_j-I)Xu\succeq 0\}$. Therefore, when $m\geq 2Pd,$ we can set $\{u_j,v_j\}$ to enumerate through $\{\tilde u_{jk}',\tilde v_{jk}\}$ and $\{\tilde u_{jk}'',-\tilde v_{jk}\}$ to achieve optimal value of (\ref{high_dim_cvx0}). With absolute value activation, the conclusion holds by replacing $D_j$ with $\text{diag}\left(\signfcn(Xu_j)\right)$ and $D_1,\ldots,D_P$ enumerate all possible sign patterns of $\text{diag}\left(\signfcn(Xu)\right).$
\end{proof}

\section{Proof in Section \ref{convergence}}\label{convergence_proof}
\subsection{Proof of Score Prediction}\label{score_pred_proof}
\subsubsection{Score Prediction for ReLU without Skip Connection}\label{relu_noskip_score}
\begin{proof}
The optimality condition for convex program (\ref{relu_noskip_formula}) is 
\begin{equation}\label{opt_cond}
0\in A^TAy+b+\beta\theta_1,
\end{equation}
where $\theta_1\in\partial \|y\|_1$. To show $y^\star$ satisfies optimality condition (\ref{opt_cond}), let $a_i$ denote the $i$th column of $A$. Check the first entry, 
\begin{equation*}
\begin{aligned}
& \quad a_1^TAy+b_1+\beta(-1)
= nvy^\star_{1}-nvy^\star_{3n}+n-\beta=0.
\end{aligned}
\end{equation*}
Check the $3n$th entry,
\begin{equation*}
\begin{aligned}
& \quad a_{3n}^TAy+b_{3n}+\beta
= -nvy^\star_{1}+nvy^\star_{3n}-n+\beta=0.
\end{aligned}
\end{equation*}
For $j$th entry with $j\not\in\{1,3n\}$, note 
\begin{equation*}
\begin{aligned}
&\quad |a_j^TAy+b_j|\\
& =|a_j^T(a_1 y^\star_{1}+a_{3n} y^\star_{3n})+b_j|\\
&=|a_j^T(a_1 y^\star_{1}-a_{1} y^\star_{3n})+b_j|\\
&=\left|\frac{\beta-n}{nv}a_j^Ta_1+b_j\right|.
\end{aligned}
\end{equation*}
Since $|b_j|\leq n-1$, by continuity, $|a_j^TAy+b_j|\leq\beta$ should hold as we decrease $\beta$ a little further to threshold $\beta_1=\max_{j\not\in\{1,3n\}} |a_j^TAy+b_j|$. Therefore, $y^\star$ is optimal.

\end{proof}
\subsubsection{Score Prediction for Absolute Value Activation without Skip Connection}\label{abs_noskip_score}
\begin{proof}
Assume without loss of generality data points are ordered as $x_1<\ldots<x_n$, then
\[b=[n,n-2,\cdots,-(n-2),n-2,n-4,\cdots,-n]\,.\]
The optimality condition  to the convex program (\ref{relu_noskip_formula}) is given by
\begin{equation}\label{sub_cond}
0\in A^T Ay+b+\beta\theta_1\,,
\end{equation}
where $\theta_1 \in \partial \|y\|_1$. To show $y^\star$ satisfies optimality condition (\ref{sub_cond}), let $a_i$ denote the $i$th column of $A$. We check the first entry
\begin{equation*}
\begin{aligned}
& \quad a_1^TAy+b_1+\beta(-1)= nvy_1-nvy_n+n-\beta=0.
\end{aligned}
\end{equation*}
We then check the last entry
\begin{equation*}
\begin{aligned}
& \quad a_n^TAy+b_n+\beta
= -nvy_1+nvy_n-n+\beta=0.
\end{aligned}
\end{equation*}
For $j$th entry with $1<j<n$, note 
\begin{equation*}
\begin{aligned}
&\quad |a_j^TAy+b_j|\\
& =|a_j^T(a_1 y_1+a_n y_n)+b_j|\\
&=|a_j^T(a_1 y_1-a_1 y_n)+b_j|\\
&=\left|\frac{\beta-n}{nv}a_j^Ta_1+b_j\right|.
\end{aligned}
\end{equation*}
Since $|b_j|\leq n-2$, by continuity, $|a_j^TAy+b_j|\leq\beta$ should hold as we decrease $\beta$ a little further to some threshold $\beta_2=\max_{j\not\in\{1,n\}} |a_j^TAy+b_j|$. Therefore, $y^\star$ satisfies (\ref{sub_cond}).
\end{proof}
\subsubsection{Score Prediction for ReLU with Skip Connection}\label{relu_skip_score}
In the convex program (\ref{relu_skip_formula}), $y=0$ is an  optimal solution when $2\beta\geq \|b\|_\infty$. Therefore, following the reconstruction procedure described in Appendix \ref{relu_skip_section}, the corresponding neural network parameter set is given by $\{W^{(1)}=0,b^{(1)}=0,W^{(2)}=0,b^{(2)}=\mu/v,V=-1/v\}$ with $\mu$ and $v$ denotes the sample mean and sample variance as described in Section \ref{convergence}. For any test data $\hat x$, the corresponding predicted score is given by
\[\hat y=V\hat x +b^{(2)}=-\frac{1}{v}(\hat x-\mu),\]
which gives the score function of Gaussian distribution with mean being sample mean and variance being sample variance.Therefore, adding skip connection would change the zero score prediction to a linear function parameterized by sample mean and variance in the large weight decay regime. 
\subsubsection{Score Prediction for Absolute Value Activation with Skip Connection}\label{abs_skip_score}
Consider convex program (\ref{abs_skip_formula}), when $\beta > \| b\|_\infty,$ $y=0$ is optimal. Following the reconstruction procedure described in Appendix \ref{abs_skip_sec}, the corresponding neural network parameter set is given by $\{W^{(1)}=0,b^{(1)}=0,W^{(2)}=0,b^{(2)}=\mu/v, V=-1/v\}$ with $\mu$ and $v$ denotes the sample mean and sample variance as described in Section \ref{convergence}. For any testing data $\hat x$, the corresponding predicted score is given by 
\[\hat y=V\hat x +b^{(2)}=-\frac{1}{v}(\hat x-\mu),\]
which is the score function of Gaussian distribution with mean being sample mean and variance being sample variance, just as the case for $\sigma$ being ReLU activation and $V\neq 0$ described in Appendix \ref{relu_skip_section}.
\subsection{Proof of Convergence Result}\label{conv_proof}
\begin{proof}
When $\beta_1<\beta\leq n$, the predicted score function is differentiable almost everywhere with least slope $0$ and largest slope $(n-\beta)/(nv)$. Then since the integrated score function is weakly  concave, Theorem \ref{conv_thm2} follows case 1 in Theorem 4.3.6 in \citep{sinho2023logconcave}.
\end{proof}
\section{Proof in Section \ref{dsm_section}}
\subsection{Proof of Theorem \ref{thm4}}\label{thm41_proof}
\subsubsection{Proof of Theorem \ref{thm4} for ReLU activation}\label{thm41_relu_proof}
Consider data matrix $x\in\mathbb{R}^{n}$. Let $m$ denote number of hidden neurons, then we have first layer weight $w\in\mathbb{R}^m,$ first layer bias $b\in\mathbb{R}^m,$ second layer weight $\alpha\in\mathbb{R}^{m}$ and second layer bias $b_0\in\mathbb{R}.$ Let $l$ denotes the label vector, i.e, $l=[\delta_1/\epsilon,\delta_2/\epsilon,\ldots,\delta_n/\epsilon]^T$. The score matching objective is reduced to
\begin{equation*}
p^\star=\min_{w,\alpha,b}\frac{1}{2}\left\|\sum_{j=1}^m (x w_j+1 b_j)_+\alpha_j+1 b_0-l\right\|_2^2+\frac{1}{2}\beta \sum_{j=1}^m (w_j^2+\alpha_j^2).
\end{equation*}
According to Lemma 2 in \cite{pilanci2020neural}, after rescaling, the above problem is equivalent to 
\begin{equation*}
\min_{\substack{w,\alpha,b\\|w_j|=1}} \frac{1}{2}\left\|\sum_{j=1}^m (x w_j+1 b_j)_+\alpha_j+1 b_0-l\right\|_2^2+\beta \sum_{j=1}^m |\alpha_j|,
\end{equation*}
which can be rewritten as 
\begin{equation}\label{dsm_1d_22}
\begin{aligned}
& \min_{\substack{w,\alpha,b,r\\|w_j|=1}} \frac{1}{2}\|r\|_2^2+\beta\sum_{j=1}^m |\alpha_j|\\
&\qquad \mbox{s.t. } r = \sum_{j=1}^m (x w_j+1 b_j)_+\alpha_j+1 b_0-l.
\end{aligned}
\end{equation}
The dual of problem (\ref{dsm_1d_22}) writes
\begin{equation*}
\begin{aligned}
d^\star=&\max_{z} -\frac{1}{2}\|z\|_2^2+z^Tl\\
& \mbox{s.t.} 
\begin{cases}
|z^T(x-1 x_i)_+|\leq \beta \quad\\ 
|z^T(-x+1 x_i)_+|\leq \beta \quad \\
z^T1 = 0
\end{cases}  \forall i=1,\ldots,n.
\end{aligned}
\end{equation*}
Note the constraint set is strictly feasible since $z=0$ always satisfies the constraints, Slater's condition holds and we get the dual problem as
\[d^\star=\min_{\substack{z_0,z_1,z_2,z_3,z_4\\ \mbox{s.t.} z_0,z_1,z_2,z_3\geq 0}} 
\frac{1}{2}\|e\|_2^2+f,\]
where $e=\sum_{i=1}^n z_{0i}(x-1 x_i)_+-\sum_{i=1}^n z_{1i}(x-1 x_i)_++\sum_{i=1}^n z_{2i}(-x+1x_i)_+-\sum_{i=1}^n z_{3i}(-x+1x_i)_++1z_4+l$ and $f=\beta(\|z_0\|_1+\|z_1\|_1+\|z_2\|_1+\|z_3\|_1)$. Simplify to get
\[\min_y \frac{1}{2}\|Ay+\bar l\|_2^2+\beta\|y\|_1,\]
with $A=[\bar A_1,\bar A_2]\in \mathbb{R}^{n\times 2n}$ where $\bar A_1=\left(I-\frac{1}{n}11^T\right)A_1, \bar A_2=\left(I-\frac{1}{n}11^T \right)A_2$ with $[A_1]_{ij}=(x_i-x_j)_+$ and $[A_2]_{ij}=(-x_i+x_j)_+$. $\bar l_j=l_j-\sum_i l_i/n$ is the mean-subtracted label vector. Once we obtain optimal solution $y^\star$ to problem (\ref{thm4formula}), we can take 
\[\begin{cases}
w_j^\star=\sqrt{|y^\star_j|}, \alpha_j^\star=-\sqrt{|y^\star_j|},b^\star_j=-\sqrt{|y^\star_j|}x_j \text{ for }j=1,\ldots,n,\\
w_j^\star=-\sqrt{|y^\star_j|}, \alpha_j^\star=-\sqrt{|y^\star_j|},b^\star_j=\sqrt{|y^\star_j|}x_{j-n} \text{ for }j=n+1,\ldots,2n,\\
b_0^\star=\frac{1}{n}1^T([A_1,A_2]y^\star+l),
\end{cases}\]
then denoising score matching objective has the same value as optimal value of convex program (\ref{thm4formula}), which indicates $p^\star=d^\star$ and the  above parameter set is optimal.
\subsubsection{Proof of Theorem \ref{thm4} for Absolute Value Activation}

\begin{proof}
Consider data matrix $x\in\mathbb{R}^{n}$. Let $m$ denote number of hidden neurons, then we have first layer weight $w\in\mathbb{R}^m,$ first layer bias $b\in\mathbb{R}^m,$ second layer weight $\alpha\in\mathbb{R}^{m}$ and second layer bias $b_0\in\mathbb{R}.$ Let $l$ denotes the label vector, i.e, $l=[\delta_1/\epsilon,\delta_2/\epsilon,\ldots,\delta_n/\epsilon]^T$. The score matching objective is reduced to
\begin{equation*}
p^\star=\min_{w,\alpha,b}\frac{1}{2}\left\|\sum_{j=1}^m (x w_j+1 b_j)_+\alpha_j+1 b_0-l\right\|_2^2+\frac{1}{2}\beta \sum_{j=1}^m (w_j^2+\alpha_j^2).
\end{equation*}
According to Lemma 2 in \cite{pilanci2020neural}, after rescaling, the above problem is equivalent to 
\begin{equation*}
\min_{\substack{w,\alpha,b\\|w_j|=1}} \frac{1}{2}\left\|\sum_{j=1}^m (x w_j+1 b_j)_+\alpha_j+1 b_0-l\right\|_2^2+\beta \sum_{j=1}^m |\alpha_j|,
\end{equation*}
which can be rewritten as 
\begin{equation}\label{dsm_1d_22_1}
\begin{aligned}
& \min_{\substack{w,\alpha,b,r\\|w_j|=1}} \frac{1}{2}\|r\|_2^2+\beta\sum_{j=1}^m |\alpha_j|\\
&\qquad \mbox{s.t. } r = \sum_{j=1}^m (x w_j+1 b_j)_+\alpha_j+1 b_0-l.
\end{aligned}
\end{equation}
The dual of problem (\ref{dsm_1d_22_1}) writes
\begin{equation*}
\begin{aligned}
d^\star=&\max_{z} -\frac{1}{2}\|z\|_2^2+z^Tl\\
& \mbox{s.t.} 
\begin{cases}
|z^T(x-1 x_i)_+|\leq \beta \quad\\ 
|z^T(-x+1 x_i)_+|\leq \beta \quad \\
z^T1 = 0
\end{cases}  \forall i=1,\ldots,n.
\end{aligned}
\end{equation*}
Note the constraint set is strictly feasible since $z=0$ always satisfies the constraints, Slater's condition holds and we get the dual problem as
\[d^\star=\min_{\substack{z_0,z_1,z_2,z_3,z_4\\ \mbox{s.t.} z_0,z_1,z_2,z_3\geq 0}} 
\frac{1}{2}\|e\|_2^2+f,\]
where $e=\sum_{i=1}^n z_{0i}(x-1 x_i)_+-\sum_{i=1}^n z_{1i}(x-1 x_i)_++\sum_{i=1}^n z_{2i}(-x+1x_i)_+-\sum_{i=1}^n z_{3i}(-x+1x_i)_++1z_4+l$ and $f=\beta(\|z_0\|_1+\|z_1\|_1+\|z_2\|_1+\|z_3\|_1)$. Simplify to get
\[\min_y \frac{1}{2}\|Ay+\bar l\|_2^2+\beta\|y\|_1.\]
where $A=\left(I-\frac{1}{n}11^T\right)A_3\in \mathbb{R}^{n\times n}$ with $[A_3]_{ij}=|x_i- x_j|,\bar l$ is the same as defined in Appendix \ref{thm41_relu_proof}. Once we obtain optimal solution $y^\star$ to problem (\ref{thm4formula}), we can take 
\[\begin{cases}
w_j^\star=\sqrt{|y^\star_j|}, \alpha_j^\star=-\sqrt{|y^\star_j|},b^\star_j=-\sqrt{|y^\star_j|}x_j \text{ for }j=1,\ldots,n,\\
w_j^\star=-\sqrt{|y^\star_j|}, \alpha_j^\star=-\sqrt{|y^\star_j|},b^\star_j=\sqrt{|y^\star_j|}x_{j-n} \text{ for }j=n+1,\ldots,2n,\\
b_0^\star=\frac{1}{n}1^T([A_1,A_2]y^\star+l),
\end{cases}\]
then denoising score matching objective has the same value as optimal value of convex program (\ref{thm4formula}), which indicates $p^\star=d^\star$ and the  above parameter set is optimal. 
\end{proof}
\subsection{Proof of Theorem \ref{thm4_2}}\label{thm42_proof}
\begin{proof}
When $X\in\mathbb{R}^{n\times d}$ for some $d>1$, when $\beta=0,$ the score matching objective can be reduced to
\begin{equation*}
p^\star=\min_{u_j,v_j} \sum_{i=1}^n \frac{1}{2}\left\|\sum_{j=1}^m(X_iu_j)_+v_j^T-L_i\right\|_2^2,
\end{equation*}
which can be rewritten as 
\begin{equation}\label{high-dm-prob}
\min_{u_j,v_j} \frac{1}{2}\left\|\sum_{j=1}^m (Xu_j)_+v_j^T-Y\right\|_F^2.
\end{equation}
Let $D'_j=\text{diag}\left(\mathbbm{1}\{Xu_j\geq 0\}\right)$, then problem (\ref{high-dm-prob}) is equivalent to
\begin{equation*}
\min_{u_j,v_j} \frac{1}{2}\left\|\sum_{j=1}^m D'_jXu_jv_j^T-Y\right\|_F^2.
\end{equation*}
Therefore, 
\begin{align*}
p^\star & =\min_{\substack{W_j=u_jv_j^T\\(2D'_j-I)Xu_j\geq 0}}\frac{1}{2}\left\|\sum_{j=1}^m D'_jXW_j-Y\right\|_F^2\nonumber\\
& \geq \min_{W_j} \frac{1}{2} \left\|\sum_{j=1}^P D_jXW_j-Y\right\|_F^2,
\end{align*}
where $D_1,\ldots,D_P$ enumerate all possible sign patterns of $\text{diag}\left(\mathbbm{1}\{Xu\geq 0\}\right).$ 
 Under Assumption \ref{multi_assum},  the construction of optimal parameter set follows  Appendix \ref{sm_nd_proof}. With absolute value activation, the same conclusion holds by replacing $D'_j$ to be $\text{diag}\left(\signfcn(Xu_j)\right)$ and $D_1,\ldots,D_P$ enumerate all possible sign patterns of $\text{diag}\left(\signfcn(Xu)\right).$
\end{proof}
\section{Details for Numerical Experiments in Section \ref{simu_sec}}\label{simu_supp_detail}
\subsection{Score Matching Fitting}\label{sm_simu_detail}
For Gaussian data experiment, the training dataset contains  $n=500$ data points sampled from standard Gaussian. For non-convex neural network training, we run 10 trials with different random parameter initiations and solve with Adam optimizer with step size $1e-2$. We train for 500 epochs. We run Langevin dynamics sampling (Algorithm \ref{alg_2}) with convex score predictor with $10^5$ data points and $T=500$ iterations, we take $\mu_0$ to be uniform distribution from $-10$ to $10$ and $\epsilon=1.$ 

For Gaussian mixture experiment, the training dataset contains two Gaussian component each containing $500$ data points, with centers at $-10$ and $10$ and both have standard variance. We take $\beta=20$.
\subsection{Denoising Score Matching Fitting}\label{dsm_simu_detail}
For spiral data simulation, we first generate $100$ data points forming a spiral as shown in the left most plot in Figure \ref{dscore_simu_}. We then add five levels of Gaussian noise with mean zero and standard deviation $[0.5,0.1,0.05,0.03,0.01]$. Thus the training data set contains $500$ noisy data points. We fit five $2d$ convex score predictors corresponding to each noise level. We solve the convex program with CVXPY \cite{diamond2016cvxpy} with MOSEK solver \cite{mosek}. The score plot corresponding to fitting our convex program with noise level $0.03.$ For annealed Langevin sampling, we sample $500$ data points in total, starting from uniform distribution on $[-10,10]$ interval. We set $\epsilon=1$ in each single Langevin process in Algorithm \ref{alg_2}. We present the sample scatter plots sequentially after sample with $0.5$ noise level score predictor for $5$ steps (Level 1), $0.1$ noise level score predictor for $5$ steps (Level 2), $0.05$ noise level score predictor for $5$ steps (Level 3), $0.03$ noise level score predictor for $5$ steps (Level 4), and finally $0.01$ noise level score predictor for $15$ steps (Level 5).
\subsubsection{Convex Reformulations of Denoising Score Matching for the Spiral Data Generation}\label{2d_exact}
\newcommand{\W}{{W}}
\newcommand{\loss}{\ell}
\newcommand{\ones}{{1}}
\newcommand{\X}{{X}}
\newcommand{\reals}{\mathbb{R}}
\newcommand{\cross}{\times}
\newcommand{\x}{x}
\newcommand{\Z}{Z}
\newcommand{\rr}{d}
\newcommand{\K}{K}
\newcommand{\rank}{\mathbf{rank}}
Since convex program (\ref{high_dim_cvx_2}) requires to iterate over all activation pattern $D_i$'s, here for easier implementation, we instead follow a variant of (\ref{high_dim_cvx_2}) which has been derived in Theorem 14 of \cite{pilanci2024complexity}, we replicate here for completeness. The formal theorem and our implementation follow Theorem \ref{thm:vectorout} below. We first state the definition of Maximum Chamber Diameter, which is used in statement of Theorem \ref{thm:vectorout} for theoretic soundness.
\begin{definition}
 We define the Maximum Chamber Diameter, denoted as \[\diam(\X):=\max_{\substack{w,v\in\mathbb{R}^d,\|w\|_2=\|v\|_2=1\\ \text{sign}(Xw)=\text{sign}(Xv)}} \|w-v\|_2.\]
\end{definition}

Consider the denoising score matching objective for a two-layer ReLU model given by
\begin{align}
    \label{eq:two_layer_relu_vector}
    p_v^*&\triangleq\min_{\W^{(1)},\W^{(2)},b} \left\|\sum_{j=1}^m\sigma(\X \W^{(1)}_j)\W^{(2)}_{j}-L\right\|_F^2 + \lambda \sum_{j=1}^m \|\W^{(1)}_{j}\|_p^2 + \|\W^{(2)}_{j}\|_p^2.
    \end{align}
    Here, the label matrix $L\in\reals^{n\times d}$ contains the $d$-dimensional noise labels as in equation (\ref{high_dim_cvx_2}), and $\W^{(1)}\in\reals^{d\times m}$, $\W^{(2)}\in\reals^{m\times d}$.
We introduce the following equivalent convex program.
\begin{align}
    \label{eq:convex_d_dim_two_layer_relu_l1_vector}
        \hat p_v \triangleq \min_{Z \in \reals^{c\times d}} \|\K\Z-L\|_F^2 + \lambda \sum_{j=1} \|Z_{ j}\|_2,
    \end{align}
where $Z_j$ is the $j$-th column of the matrix $\Z$ and $c=\binom{n}{d-1}$.
\begin{theorem}
    \label{thm:vectorout}
    Define the matrix $K$ as follows
    \begin{align*}
        \K_{ij} =
                 \frac{\big( x_i \wedge x_{j_1}\wedge\cdots\wedge x_{j_{\rr-1}}\big)_+}{\|\x_{j_1} \wedge\,...\, \wedge x_{j_{\rr-1}} \|_p}\,,
    \end{align*}
    where the multi-index $j=(j_1,...,j_{\rr-1})$ is over all combinations of $d-1$ rows. It holds that
    \begin{itemize}
    \item when $p=1$, the convex problem \eqref{eq:convex_d_dim_two_layer_relu_l1_vector} is equivalent to the non-convex problem \eqref{eq:two_layer_relu_vector}, i.e., $p_v^*=\hat p_v$.
    \item when $p=2$, the convex problem \eqref{eq:convex_d_dim_two_layer_relu_l1_vector} is a $\frac{1}{1-\epsilon}$ approximation of the non-convex problem \eqref{eq:two_layer_relu_vector}, i.e., $p_v^*\le \hat p_v \le \frac{1}{1-\epsilon} p_v^*$, where $\epsilon\in(0,1)$ is an upper-bound on the maximum chamber diameter $\diam(\X)$. 
    \end{itemize}
    An neural network achieving the above approximation bound can be constructed as follows:
    \begin{align}
        \label{eq:optimal_d_dim_vector}
        f(x) = \sum_{j}Z^*_{ j}\frac{\big( x_i \wedge x_{j_1}\wedge\cdots\wedge x_{j_{\rr-1}}\big)_+}{\|\x_{j_1} \wedge\,...\, \wedge x_{j_{\rr-1}} \|_p}\,,
    \end{align}
    where $Z^*$ is an optimal solution to \eqref{eq:convex_d_dim_two_layer_relu_l1_vector}.
\end{theorem}
See Theorem 14 in \cite{pilanci2024complexity} for proof. In our simulation in Section \ref{dsm_simu_sec}, we implement convex program (\ref{eq:convex_d_dim_two_layer_relu_l1_vector}), solve with CVXPY \cite{diamond2016cvxpy}, and get score prediction from equation (\ref{eq:optimal_d_dim_vector}).

\section{Additional Simulation Results}\label{simu_supp}
In this section, we give more simulation results besides those discussed in main text in Section \ref{simu_sec}. In Section \ref{score_simu_supp} we show simulation results for score matching tasks with more neural network types and data distributions. In section \ref{dscore_simu_supp} we show more simulation results for denoising score matching tasks.
\subsection{Supplemental Simulation for Score Matching Fitting}\label{score_simu_supp}
\begin{figure*}[ht!]
 \makebox[\textwidth][c]{\includegraphics[width=1.0\linewidth]{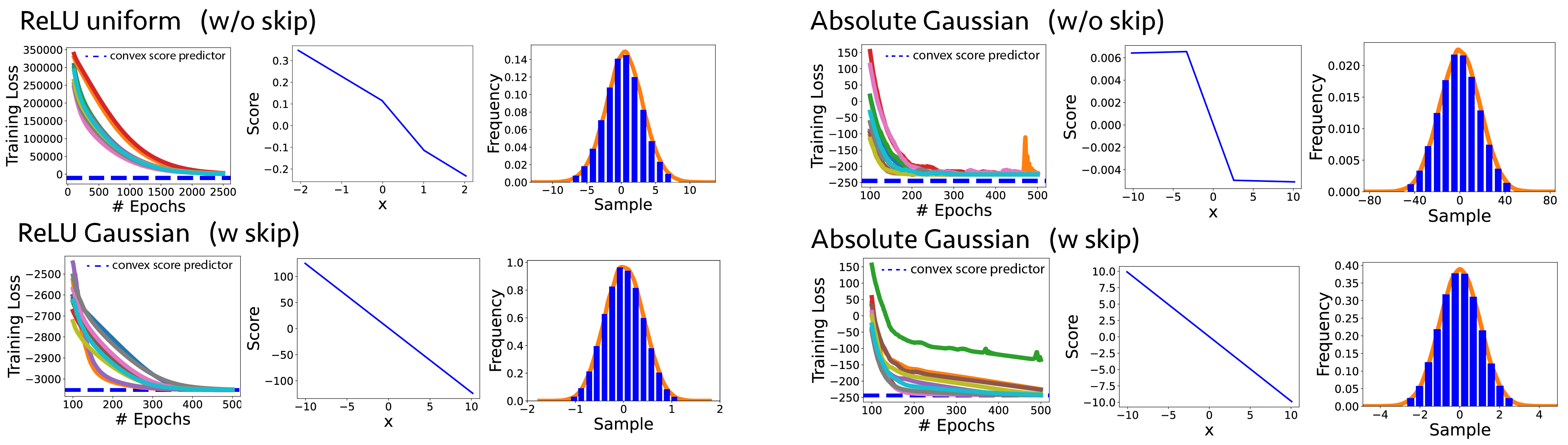}}
\caption{Simulation results for 
 score matching tasks with two-layer neural network. The left subplots for all four categories show training loss where the dashed blue lines indicate loss of convex score predictor. The middle plots show score prediction by convex score predictor. The right plots show sampling histograms via plain Langevin process with convex score predictor.  See Appendix \ref{score_simu_supp} for details.}\label{sm_more_simu}
\end{figure*}
Here we verify our findings in Section \ref{convergence} with more experiments. The upper left plot in Figure \ref{sm_more_simu} shows results for two-layer ReLU network without skip connection and with training data of uniform distribution on range $[0,1]$. Here we still set $\beta=\|b\|_\infty-1$ as in Section \ref{score_simu_sec}. Our theoretic analysis in Section \ref{convergence} reveals that for this $\beta$ value, the predicted score corresponding to Gaussian distribution characterized with sample mean and sample variance, which is corroborated by our simulation results here, i.e., the mid-subplot shows score function contained in left plot in Figure \ref{relu_abs}. The upper right plot follows same experimental setup except that here we experiment with absolute value activation instead of ReLU and the training data is standard normal. The predicted score is aligned with our theoretic derivation in right plot in Figure \ref{relu_abs}.

The bottom left and bottom right plots are for networks with skip connection. We set $\beta=\|b\|_\infty+1$. Our theory in Section \ref{convergence} concludes that for this $\beta$ value, NNs without skip connection would predict zero scores while NNs with skip connection predict linear score corresponding to Gaussian distribution chractrized with sample mean and sample variance, which is supported by our simulation results.

For non-convex training, we run 10 trials with different random parameter initiations. It can be observed that our convex program always solves the training problem globally. Note for absolute value activation NNs (top right and bottom right plots), non-convex training sometimes sticks with local optimality, reflected by the gap of convergence value between non-convex training and our convex fitting.
\begin{figure*}[ht!]
 \makebox[\textwidth][c]{\includegraphics[width=1.0\linewidth]{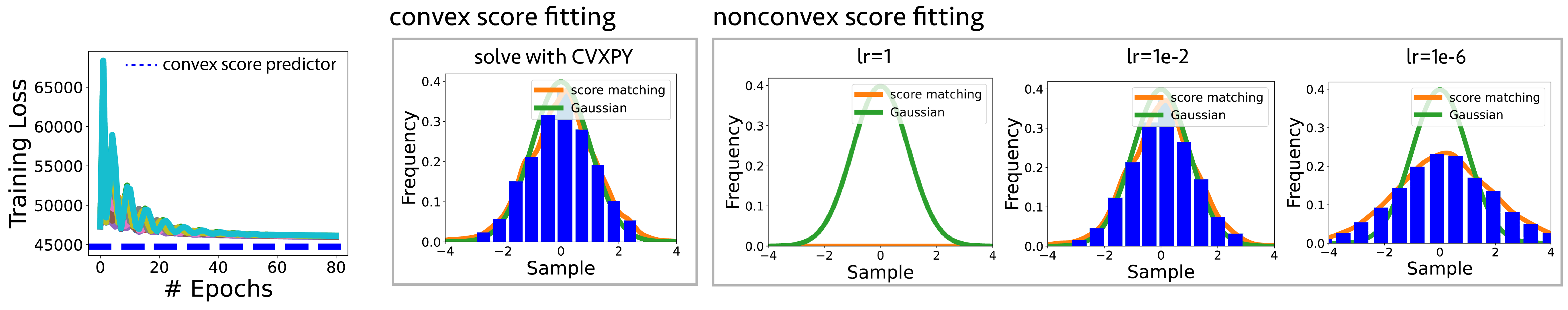}}
\caption{Simulation results for denoising score matching tasks with two-layer ReLU neural network. The left plot shows training loss where the dashed blue line indicates loss of convex score predictor (\ref{thm4formula}). The second plot shows sampling histogram via annealed Langevin process with convex score predictor. The third, fourth, and fifth plots show sampling histograms via annealed Langevin process with  non-convex score predictors trained with learning rates $1,1e-2,1e-6$ respectively. The ground truth distribution is standard Gaussian, which is recovered by our model.}\label{dsm_simu_supp}
\end{figure*}
\subsection{Supplemental Simulation for Denoising Score Matching Fitting}\label{dscore_simu_supp}
For experiment in Figure \ref{dsm_simu_supp}, training data is standard Gaussian and $\beta=0.5$ is adopted. we take ten noise levels with standard deviation $[\sigma_1,\ldots,\sigma_L]$ being the uniform grid from $1$ to $0.01$. For each noise level, we sample $10$ steps. Initial sample points follow uniform distribution in range $[-1,1]$. The non-convex training  uses Adam optimizer and takes $200$ epochs. Left most plot in Figure \ref{dsm_simu_supp} shows the training loss of $10$ non-convex fittings with stepsize $lr=1e-2$ and our convex fitting. It can be observed that our convex fitting achieves lower training loss than all non-convex fittings. The second plot in Figure \ref{dsm_simu_supp} shows annealed Langevin sampling histogram using our convex score predictor, which captures the underline Gaussian distribution. The right three plots show annealed Langevin sampling histograms with non-convex fitted score predictor trained with different learning rates. With $lr=1,$ training loss diverges and thus the predict score diverges from the true score of training data. Thus the sample histogram diverges from Gaussian. With $lr=1e-2,$ non-convex fitted NN recognizes the desired distribution while with $lr=1e-6,$ the NN is not trained enough thus the sampling results resemble Gaussian to some extent but not accurately. These results show that non-convex fitted score predictor is sometimes unstable due to training hyperparameter setting while convex fitted score predictor is usually much more reliable and thus gains empirical advantage.
\end{appendix}

\end{document}